\documentclass[10pt,twocolumn,twoside]{style_and_bib_files/IEEEtran}
\usepackage[cmex10]{amsmath}
\usepackage{amssymb,  amsfonts, amsthm, cite, color, graphicx, hyperref, array, verbatim, algorithm, algorithmic}
\usepackage{style_and_bib_files/mysymbol}
\usepackage[english]{babel}
\definecolor{darkred}{rgb}{0.6,0,0}
\usepackage[short,nodayofweek]{datetime}

\let\labelindent\relax
\usepackage{enumitem}
\setlist[enumerate,1]{%
  label=\arabic*.,
}

\newlist{inlinelist}{enumerate*}{1}
\setlist*[inlinelist,1]{%
  label=(\roman*),
}

\newtheorem{theorem}{Theorem}[section]
\newtheorem{lem}[theorem]{Lem}
\newtheorem{prop}[theorem]{Proposition}

\newtheorem{rem}[theorem]{Remark}

\newtheorem{ass}[theorem]{Assumption}

\graphicspath{{figures/}}

\newcommand{\transit}{\operatornamewithlimits{\it transit}}
\newcommand{\busy}{\operatornamewithlimits{\it busy}}
\def\proj{\mathsf{\Pi}}
\def \Sym {{\rm Sym}}

\def\ccalCS{{\ccalS_C}}
\def\ccalCU{{\ccalU_C}}
\def\PhiC{{\bbPhi_C}}
\def \edRVI{{\epsilon\text{-}\delta\text{RVI}}}

\allowdisplaybreaks

\title{Distributed Hierarchical Control for State Estimation With Robotic Sensor Networks}\date{}\author{Charles~Freundlich,~\IEEEmembership{Student~Member,~IEEE,}~Yan~Zhang,~\IEEEmembership{Student~Member,~IEEE,}
and~Michael~M.~Zavlanos,~\IEEEmembership{Member,~IEEE}
\thanks{%
Charles Freundlich, Yan Zhang, and Michael M. Zavlanos are with the Dept. of Mechanical Engineering and Materials Science, Duke University, Durham, NC 27708, USA {\tt\footnotesize \{charles.freundlich, yan.zhang2, michael.zavlanos\}@duke.edu}. This work is supported in part by the NSF award CNS \#1302284. A preliminary version of this work can be found in \cite{freundlich15acc}
}}

\begin{document}
\maketitle
\thispagestyle{empty}
\pagestyle{empty}

\begin{abstract}This paper addresses active state estimation with a team of robotic sensors.
The states to be estimated are represented by spatially distributed, uncorrelated, stationary vectors.
Given a prior belief on the geographic locations of the states, we cluster the states in moderately sized groups and propose a new hierarchical Dynamic Programming (DP) framework to compute optimal sensing policies for each cluster that mitigates the computational cost of planning optimal policies in the combined belief space. Then, we develop a decentralized assignment algorithm that dynamically allocates clusters to robots based on the pre-computed optimal policies at each cluster.
The integrated distributed state estimation framework is optimal at the cluster level but also scales very well to large numbers of states and robot sensors. 
We demonstrate efficiency of the proposed method in both simulations and real-world experiments using stereoscopic vision sensors.\end{abstract}


\section{Introduction}\label{sec:intro}\PARstart{R}{obotic} sensors rely on mobility to gather information.
Information acquisition can be subtask in a more complex robotic mission such
as SLAM, or the end goal in, e.g., geostatistical surveying, environmental sampling, or mapping missions.
%
The goal of
this paper is to determine how a team of robots should
collect information so that the aggregate uncertainty in a finite
collection of hidden state vectors is minimized. Specifically,
given prior beliefs of the geographic location of the hidden
states, we seek an optimal sequence of observations which,
when fused with the prior beliefs, minimizes the estimation
uncertainty. This problem is known in the mobile robotics
literature as distributed state estimation.
 %
The approach presented herein is similar to recent advancements in linear-Gaussian active sensing that operate in pose-covariance space \cite{atanasov14,agha15}, except that here we propose a novel belief-space discretization that admits exact value iteration and can be readily incorporated in a hierarchical multi-robot controller, enabling decentralized information acquisition of very large collections of hidden states by large teams of robots.


A typical approach to active state estimation is to employ gradient descent methods to generate sensor trajectories and sequences of associated state observations that minimize an information theoretic objective of interest, such as the trace of the covariance matrix.
This is the approach followed, e.g., in mobile target tracking \cite{chung06,freundlich13cdc}, sparse landmark localization \cite{freundlich15acc, freundlich13icra,vanderhook15}, and active SLAM \cite{carrillo15,atanasov15,agha15}.
Recently, \cite{carrillo15} have shown that information-theoretic objectives may fail even to be monotonic in many active sensing tasks, removing performance guarantees for greedy control.


When planning multiple observations, the robot may need to reason over the combinatorial set of future probability distributions (pdfs), efficiently represent them, and solve Bellman's equation.
This is known as the ``belief representation problem.''
This problem, in the context of information acquisition, has received a great deal of attention recently, both when the robot state is observable\cite{ryan2010particle,adurthi14,atanasov14,hollinger14sampling,LeNy09trajectory} and when it is only partially observable\cite{atanasov14nonmyopic,spaan14,agha2014firm,agha15,kurniawati2016online,seiler15,bai14,omidshafiei15,oliehoek16}.
The most widely used approach is to grow a tree with a prior distribution of the hidden states at the root, sampling in this way several sequences of possible future observations to obtain the set of reachable belief states at the leaves.
Once the tree has been constructed, one simply selects the leaf with the lowest cost and traces it back to the root to obtain the optimal policy.
The horizon length can be set \emph{a priori} \cite{ryan2010particle,adurthi14,atanasov14,atanasov14nonmyopic} or, e.g., defined implicitly by a budget \cite{hollinger14sampling}.
{
Note that nonmyopic active sensing problems, such as the problem addressed in this paper, implicitly exploit \emph{a priori} knowledge rather than exploring the environment to discover new features.
Exploration can be included as a first step as in \cite{atanasov14nonmyopic}, and incorporating this step within our approach is a subject of further research.
}

Two approaches that are fundamentally different from choosing a dynamic programming horizon are
\begin{inlinelist}
\item to represent the reachable belief space with a finite set in a clever way, typically by making an assumption about the family of distributions of the hidden states\cite{spaan14,freundlich15acc,atanasov14,agha2014firm,vanderhook15,LeNy09trajectory}, or
\item to avoid this representation problem altogether by working in policy space \cite{kurniawati2016online,seiler15,bai14}.
\end{inlinelist}
With respect to the latter, \cite{bai14} defines a \emph{generalized policy graph}, which nonetheless relies on belief space sampling.
In fact, some kind of sampling is at the core of all point-based approaches stemming from the seminal paper \cite{kurniawati08}.

Non-myopic active sensing for teams of decentralized robots often leads to decentralized Partially Observable Markov Decision Process POMDPs (dec-POMDPs) \cite{omidshafiei15, oliehoek16}.
To our knowledge, decentralization in this context refers to the execution of the planner, as its computation is always done offline at a central location due to the non-separable nature of the value function.
Recently, \cite{atanasov16} used sequential planning for decentralized active SLAM.
Common in the vast majority of approaches discussed above is that mobile sensor planning relies on sampling-based strategies, e.g., forward search, for partially \cite{atanasov14nonmyopic,spaan14,agha2014firm,agha15,omidshafiei15, oliehoek16,kurniawati2016online,seiler15,bai14,spaan14}
and fully observable \cite{atanasov14,LeNy09trajectory} robots, or belief space sampling in the policy domain \cite{kurniawati2016online,seiler15,bai14}.
Sampling-based approaches will typically run into scalability issues due to one or more of the following reasons:
\begin{inlinelist}
\item sparsity of information in the environment, which forces longer planning horizons,
\item high dimensional unknowns, which make observation sampling inefficient, or
\item large teams of robots, which significantly increase the size of the action space.
\end{inlinelist}
To design an algorithm that avoids these pitfalls, in this paper we introduce a hierarchical approach 
%
%
that decomposes the set of $M$ states into $P\ll M$ clusters, designs optimal controllers for each cluster, and then allocates those controllers among the $N$ robots. Specifically, for every hidden state that needs to be estimated,
we define a local Dynamic Program (DP) in the joint state-space
of robot positions and state uncertainties that determines
robot paths and associated sequences of state observations that
collectively minimize the estimation uncertainty. Then, we divide the
collection of hidden states into clusters based on a prior belief
of their geographic locations, and, for each cluster, we define a
second DP that determines how far along the local optimal trajectories
the robot should travel before transitioning to estimating
the next hidden state within the cluster. Finally, a distributed
assignment algorithm is used to dynamically allocate controllers
to the robot team from the set of optimal control policies for every
cluster.
{
At the cluster level, the problem that we solve can be considered a generalization of the TSP in that the robot must observe high dimensional unknowns over an infinite horizon at each site.
This is important to note because our method incurs the same computational complexity as the TSP, $O(2^K)$, where $K$ is the number of hidden states in the cluster.
Moreover, the greedy extension of our method to multiple robots and multiple clusters only contributes constant complexity in the size of the robot team and in the number of clusters.
In this way, our approach represents an approximation algorithm to a broad class multi vehicle generalized TSPs for which no good approximation exists.
While we are not able to guarantee an optimality gap, we are able to bound computational complexity by limiting the cluster size, effectively ``dividing and conquering'' the hard problem into a number of tractable subproblems that can be solved exactly.
}
We are not aware of any other non myopic method in the literature that can handle as many hidden states and robot sensors.
We also illustrate these claims with experiments on real robots, which are a contribution in and of themselves as the first to demonstrate experimental multi-robot active sensing using stereo vision.

The paper is organized as follows: In Section~\ref{sec:pf}, we formulate the distributed state estimation problem addressed in this paper. 
In Sections~\ref{sec:st} and \ref{sec:mt} we propose a local and a cluster DP to obtain controllers that are optimal at the cluster level.  
Section \ref{sec:many} presents the distributed auction mechanism that can efficiently allocate clusters to the robots in real-time.
In Section~\ref{sec:sim}, we simulate large robot teams carrying stereo vision rigs that localize hundreds of sparse landmarks.
In Section~\ref{sec:exp}, we report experiments on a team of two robots localizing eight landmarks.

\section{Problem Formulation}\label{sec:pf}
Consider a team of $N$ mobile robots tasked with estimating $M$ hidden state vectors $\set{\bbx_i }_{i=1}^M\subset \reals^n$.
Let $\ccalW$ and $\ccalU$ denote the configuration and action spaces of a single robot, and let $\bbphi \colon \ccalW \times \ccalU \to \ccalW$ denote its (possibly nonholonomic) dynamical model.
{
In this paper, the set $\ccalW$ can be any finite discretization of the robot's configuration space, as long as transitions in that space can be assumed to be deterministic.
}
Assume that the robots have access to a Normal prior distribution over the hidden states, with means $\set{\hbbx_{i,0}}_{i=1}^M$ and covariances $\set{\bbSigma_{i,0}}_{i=1}^M$.
Let $\bby_{ij}$ denote the observation of $\bbx_i$ by robot $j$ that is corrupted by zero mean Gaussian noise, specifically,
$
\bby_{ij} = \bbx_i + \bbnu_{ij},
$
where $\bbnu_{ij} \sim \normal{ \bb0, \bbQ_{ij}}.$
We assume that we have a model of $\bbQ_{ij}$ that depends on both the state of the robot and the hidden vector.
{
We will thus write $\bbQ_{ij}$ as a (possibly discontinuous) mapping $\bbQ \colon \reals^n \times \ccalW \to \Sym_{++}(n, \reals),$ where $\Sym(n, \reals)$ denotes the set of symmetric matrices and the subscript denotes the restriction of this set to the set of positive definite matrices.
}
Collectively, the team acquires a sequence of observations $\set{\bby_{ij,k}}$ with measurement error covariances $\set{\bbQ_{ij,k}}$ from various vantage points along controlled trajectories $\set{\bbp_{j,k}} \subset \ccalW$, where $k$ denotes a time index.
Hereafter, we will sometimes write $\bbQ(\bbx_i, \bbp_{j,k})$ instead of $\bbQ_{ij, k}$ to emphasize that $\bbQ$ is actually a function of $\bbx_i$ and $\bbp_{j,k}.$

Our goal is to minimize the variance in all hidden vectors as well as the distance the robots need to travel over the course of the controlled trajectories.
We denote by $\bbpsi \colon \ccalU \to \reals_{+}$ a metric that measures the distance an agent needs to travel as a result of actions in $\ccalU$.
Then, given a parameter $\rho  \in [0,1]$, we find a sequence of control inputs $\set{\bbu_{j,k}} \subset \ccalU$ that solve
\begin{subequations}\label{eq:orig_prob}
\begin{align}
&\max_{\set{\bbu_{j,k}} }   \sum_{k \in \mbN}  \! \gamma^k \bigg[(1\!-\! \rho)\! \sum_{i=1}^M    \! \sqrt{ \bbt\bbr \!\left(\bbSigma_{i,k} \!-\!\bbSigma_{i,k+1} \right)\! } - \!\rho \!\sum_{j=1}^N \bbpsi(\bbu_{j,k}) \bigg]
 \label{eq:reward_function}\\
&\st \; \bbp_{j,k+1} = \bbphi \left( \bbp_{j,k},\bbu_{j,k} \right) \; \forall j = 1, \dots, N
\label{eq:robot_dynamics}
 \\
& \bbSigma_{i,k+1} =
\begin{cases}
 \left(  \bbSigma_{i,k}^{-1} + \bbQ_{ij,k}^{-1} \right)^{-1} &\text{robot $j$ observes $\bbx_i$}\\
   \bbSigma_{i,k} &\text{else}
\end{cases}, \label{eq:covariance_dynamics}
\end{align}
\end{subequations}
with initial conditions $\set{\bbp_{j,0}}_{j=1}^N$ and prior error covariance $\set{\bbSigma_{i,0}}_{i=1}^M $.
For each stage $k$, the expression inside the square brackets in \eqref{eq:reward_function} is a trade off between variance reduction and distance traveled.
The parameter $\rho$ controls this tradeoff.
The square root in this expression is used to compare equal units.
Setting the discount factor $\gamma <1$ ensures that the value function for the infinite horizon problem remains finite.
Equation \eqref{eq:robot_dynamics} explicitly constrains the robot poses by the dynamics $\bbphi$, while equation \eqref{eq:covariance_dynamics} constraints the covariance dynamics by the Kalman Filter (KF) update for stationary hidden states.

The developments in the remainder of this paper toward solving problem~\eqref{eq:orig_prob} rely on the following assumptions: 
\begin{ass}
In this work we assume that the hidden state vectors are sparse and so are the observations.
This means that even if the robot passively observes multiple hidden vectors at once, our plans are only optimal with respect to reducing the uncertainty of one at a time.
\end{ass}
\begin{ass}
We assume that we have access to a noisy prior distribution $\ccalN (\hbbx_{i,0}, \bbSigma_{i,0})$ for each hidden vector.
We assume that the hidden vector is a stationary process, thus we use the prior mean (along with the sensor configuration) to evaluate the observation uncertainty for all time $k \ge 0$.
Similar to the formulation in \cite{atanasov14,LeNy09trajectory}, this means that the dynamics of the error covariance matrices are deterministic, cf. \eqref{eq:orig_prob}.
This also implies that $\bb0 \preceq \bbSigma_{i,k+1}\preceq\bbSigma_{i,k}$ under KF dynamics.
\end{ass}

\section{Localization of a Single Target}\label{sec:st}
In this section, we propose a method to solve problem~\eqref{eq:orig_prob} for $N=M=1,$ thus we drop the references to $i$ and $j$.
We call this the \emph{local DP}. 
To construct the local state-space, we approximate the space of reachable covariances by a finite set $\ccalC$.
Then, we define the local state space to be the product space $\ccalS \triangleq \ccalW \times \ccalC.$
%
We discuss the specifics of designing $\ccalC$ in Section~\ref{sec:cog}.
A state $\bbs_{k} \triangleq (\bbp_{k}, \bbSigma_{k}) \in \ccalS$ is reachable from $\bbs_{k-1} \in \ccalS$ if there exists a control input $\bbu \in \ccalU$ 
that satisfies the joint dynamical equation $\bbPhi \colon \ccalS \times \ccalU \to \ccalS$, given by
\begin{align}\label{eq:joint_dynamics}
\bbs_k \!=\!
 \bbPhi(\bbs_{k\! -\! 1}, \bbu) \!=\! \left(
  \bbphi \left( \bbp_{k\! -\! 1} ,\bbu \right) 
  ,
\proj_{\ccalC}   \left[  \bbSigma^{\! -1}_{k-1}  + \bbQ_{k\! - \!1} ^{- \!1} \right]\!^{- \!1}
\right)
,
\end{align}
where $\proj_{\ccalC} \colon \Sym_{++}(n, \reals) \to \ccalC$.
We give exact details of this projection in Section~\ref{sec:cog} and provide an example of the state-space transitions in Fig.~\ref{fig:discrete_actions}~(a).
Projection in \eqref{eq:joint_dynamics} is necessary because we require that $\bbSigma_k \in \ccalC$, which is a finite subset of $\Sym_{++}(n, \reals)$.
The function $\bbPhi$ constitutes the transition function for the local DP.
Then, denote by $R \colon \ccalS \times \ccalU \to \reals$ the instantaneous reward from problem \eqref{eq:orig_prob}, given by
\begin{align}\label{eq:reward}
R({\bbs_k},\bbu) &\triangleq  (1-\rho) \left(\bbt\bbr \left[\bbSigma_{k-1}- \bbSigma_{k} \right] \right)^{1/2}  -\rho\bbpsi(\bbu).
\end{align}
In the remainder of this section, we will design a state space that is small enough for exact value leading to a desired stationary optimal policy $\bbmu^* : \ccalS \to  \ccalU$.
{
In particular, a \emph{stationary optimal policy} is one that depends only on the current state of the system.
}

\subsection{The Uncertainty State-Space and Transition Function}\label{sec:cog}
In this section, we discretize $\Sym_{++}(n,\reals)$ to design the finite set $\ccalC.$
We emphasize that optimally sampling bounded subsets of $\Sym(n,\reals)$ is an interesting and deep problem \cite{hardin2004discretizing}, and we do not provide a general framework.
Our method works well for representing a specific bounded region of $\Sym(n , \reals),$ which we call the \emph{reachable covariance matrices} for the problem described herein.

\begin{figure}[t]
    \centering
    \begin{tabular}{c c}
    \includegraphics[width=2.6cm]{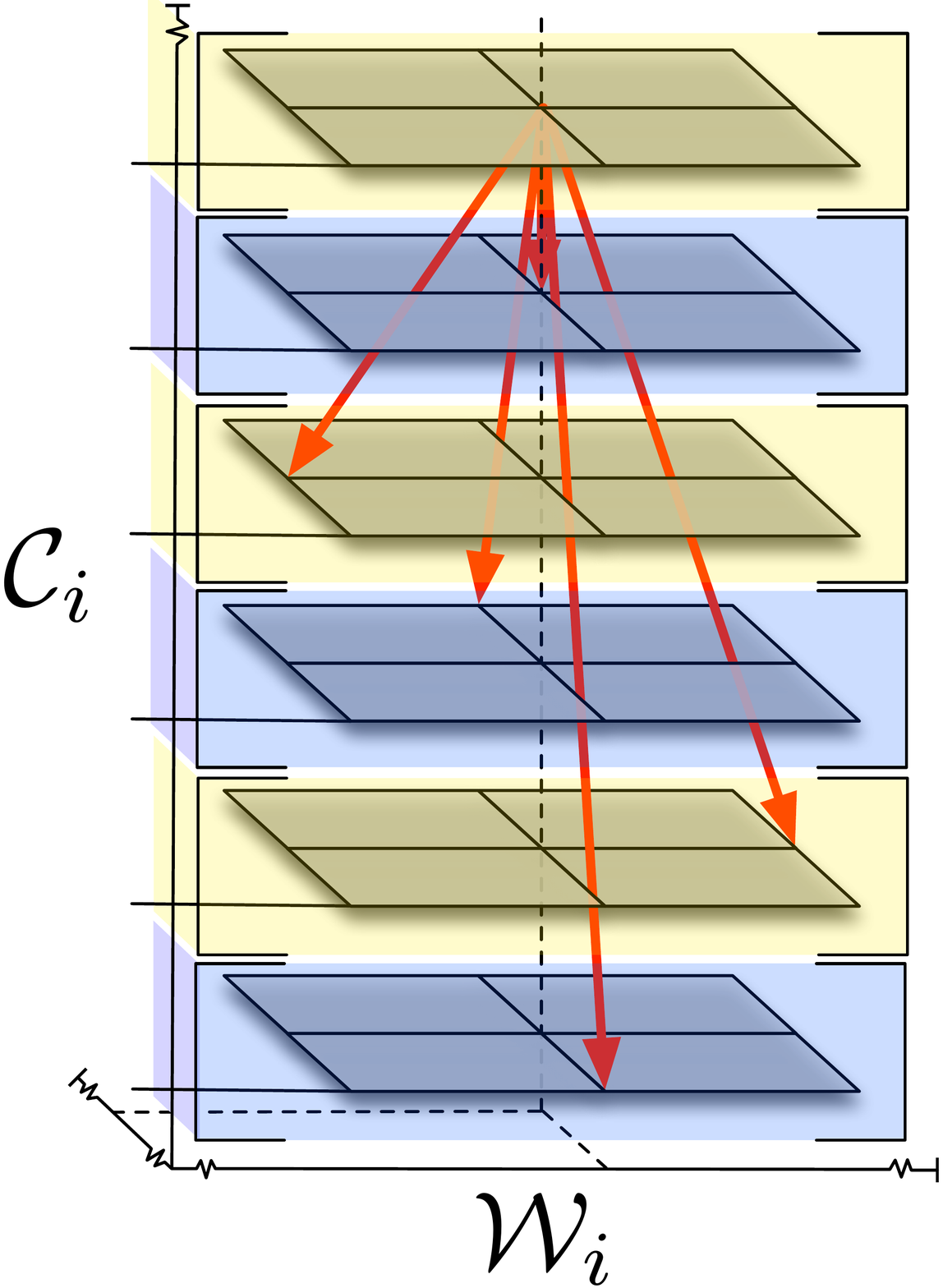} 
    &
    \includegraphics[width=3.5cm]{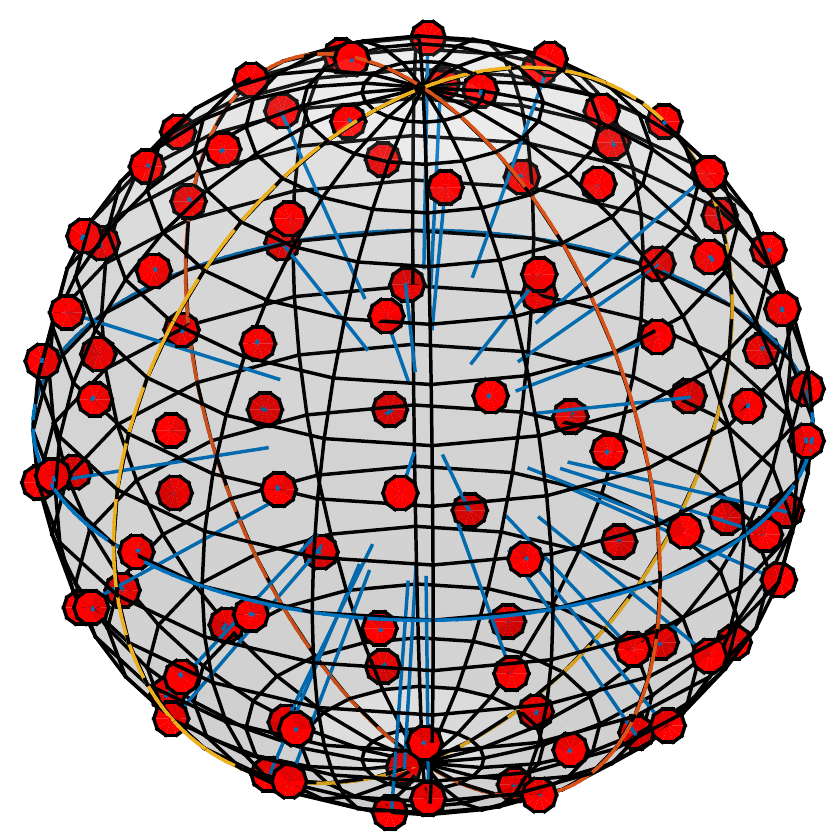} 
    \\
    (a) & (b)
    \end{tabular}
    \caption{
        (a) An illustration of the state-space and transition for a local DP where $\ccalW$ is a NWSE grid and actions are to take an image and move north, south, east, west, or remain stationary.
        Red lines are drawn to represent each action, showing the transition of the state both in $\ccalW$ and in $\ccalC$.
        The vertical axis represents the discrete nature three dimensional $\ccalC$ using colored regions, which are each represented by a single matrix in $\ccalC$.
        (b) A fifty point discretization of $\mbR\mbP^{2}$ is represented by allowing 100 simulated charges confined to the 2-sphere $\mbS^2$ to come to equilibrium and discarding the points below an arbitrary equator.
Charges are red dots, and the net electric force on each is plotted in blue.
    }
    \label{fig:discrete_actions}
    \vspace{-0.5cm}
\end{figure}

We begin by the following lemmas; proofs are omitted due to space limitations.
\begin{lem}\label{lem:pre-fusion}
Let $\bbI$ denote the $n \times n$ identity matrix and let $\bbC \in \Sym_{++}(n,\reals)$.  Then, $\bbI-(\bbI+\bbC)^{-1} \in \Sym_{++}(n,\reals)$.
\end{lem}

\begin{lem}[Lemma 2.7, \cite{trace_ineq}] \label{lem:fusion}
Let $n \in \mbN$.
If $\bbA,\bbB \in \Sym_{++}(n,\reals)$, then 
$
{\bf tr} \left( \bbA^{-1} +\bbB^{-1} \right)^{-1} < {\bf tr} \bbA.
$
\end{lem}

The first implication of Lemma~\ref{lem:fusion} is that the trace of the largest instantaneous covariance bounds the maximum eigenvalue of the reachable covariance matrices.
Define this bound as
{
\begin{equation}
\lambda_\text{max} \le \max_{ \bbp \in \ccalW} \; \bbt\bbr \;\bbQ (\hbbx, \bbp) ,
\end{equation}
where we have written the measurement covariance as a function of the prior mean of the state estimate $\hbbx \in \reals^n$, which is fixed during the planning phase, and the configuration of the sensor $\bbp \in \ccalW.$
}
Lemma~\ref{lem:fusion} also implies that the trace is decreasing with additional independent measurements.
Therefore, the set of reachable covariances is bounded.
Lemma ~\ref{lem:fusion} and the resulting bound can be used to obtain a discretization of the set of possible maximum eigenvalues for the reachable covariance matrices.
In particular, we define the logspace set of maximum eigenvalues
\begin{equation} \label{eq:mu}
\ccalL  \triangleq \set{   {\lambda_\text{max}} e^{\kappa_{\ccalL}  \left( i-N_{\ccalL} \right)/N_{\ccalL}} \mid i=1, \dots ,N_{\ccalL} }\subset \reals_{++},
\end{equation}
where $N_{\ccalL}$ is the cardinality of $\ccalL$ and $\kappa_{\ccalL}$ is a sampling gain that controls how clustered the samples are toward zero.
Note that ${\lambda_\text{max}}$ is the maximal element of $\ccalL$.
In \eqref{eq:mu}, we sample in logspace as a heuristic; we have found empirically that the maximum eigenvalues of the filtered covariance matrices accumulate near zero.

To obtain a scalable discretization of the space of covariance matrices $\Sym_{++}(n , \reals),$ we assume that $\lambda_\text{max} (\bbSigma)$ and its corresponding eigenvector are more important than any one of the other eigenvalues and eigenvectors.
In particular, we assume that, in comparison to $\lambda_\text{max},$ 
the other eigenvalues are roughly equal,
thus all other eigenvalues can be parameterized by a number in the half open interval $\alpha \in (0,1] \subset \reals$ such that $\lambda_i \approx \alpha \lambda_\text{max}$ for all $i = 2, \dots, n.$
This choice alleviates the need to independently consider all possible combinations of eigenvectors corresponding to the nonprincipal eigenvalues.
Define the set of ratios, which can be thought of as the set of possible inverse condition numbers, as
\begin{equation} \label{eq:ratiovec}
{\ccalA } \triangleq \set{ e^{\kappa_{\ccalA }  \left( i-N_{\ccalA } \right)/N_{\ccalA }} \mid i =1, \dots, N_{\ccalA } }\subset (0,1].
\end{equation}
In \eqref{eq:ratiovec}, $N_{\ccalA }$ is the number of eigenvalue ratios we sample and $\kappa_{\ccalA }$ is a sampling gain that controls the ellipticity of the confidence region associated with $(\lambda,\alpha) \in \ccalL\times\ccalA.$
Again, the logspace discretization is a heuristic based on experience; KF error covariance matrices produced using robotic sensors are typically cigar-shaped, i.e., dominated by uncertainty in the direction of the principal eigenvector.


The set of possible principal eigenvectors is equivalent to the set of lines passing through the origin, known as the real projective space $\mbR \mbP^{n-1}\!\!\!\!\!.\;$
Let $N_\ccalT$ be the number of samples needed to capture these possible directions for the principal eigenvalue.
Because $\mbR \mbP^{n-1}$can be formed by identifying antipodal points on any sphere, this problem can be approximately solved by placing $2N_\ccalT$ point charges on a sphere of radius $\sqrt{{\lambda_\text{max}}},$ allowing them to move until the ``electrostatic forces'' among them come to equilibrium, cutting the sphere along any equator, discarding one of the hemispheres, and saving the unit directions to each ``charge'' location on the other hemisphere.
It is straightforward to build such a simulation, and for brevity we do not provide the specifics here.
The result for 100 charges on the unit 2-sphere is shown in Fig.~\ref{fig:discrete_actions}~(b).

To force the sampling density to be consistent, defined in terms of the surface area of the sphere used in the simulation, we create a set of sets $\set{\ccalT_\lambda \mid \lambda \in \ccalL}$; each element $\ccalT_\lambda$ is the set of unit vectors produced by the simulation using $\sqrt{\lambda}$ as the radius of the sphere.
The number of elements in the largest, i.e., the set with the most elements, of these sets $\ccalT_{\lambda_\text{max}}$ is set to a user-specified number $N_{\ccalT_{\lambda_\text{max}}}$.
The remaining sets $\set{ \ccalT_\lambda \mid \lambda \in \ccalL}$ correspond to the other $N_{\ccalL} -1$ possible principal eigenvalues and, since ${\lambda_\text{max}}$ is the maximal element of $\ccalL,$ the sets $\set{\ccalT_\lambda \mid \lambda \neq {\lambda_\text{max}}}$ must have fewer elements than $\ccalT_{\lambda_\text{max}}$ so that the sampling density is the same.
In particular, for some $\lambda \in \ccalL,$ the set $ \ccalT_\lambda$ has
$
\ceil{
\frac{\lambda}{{\lambda_\text{max}}}
N_{\ccalT_{\lambda_\text{max}}} 
}$ elements, where the ceiling function is used to ensure that the number of elements in $\ccalT_\lambda$ is positive.

Using the sets $\ccalL, \ccalA,$ and $ \set{\ccalT_\lambda\mid \lambda \in \ccalL},$ we can create a discretization of $\Sym_{++} (n, \reals)$.
The result, interpreted geometrically, is $N_\ccalL$ concentric sets of cigar-shaped confidence ellipsoids with a variety of major diameters, defined by $\lambda \in \ccalL,$ ellipticities $\alpha \in \ccalA,$ and orientations $\bbu \in \ccalT_\lambda.$
The full covariance space can thus be described with a map $f \colon \reals^{2+n} \to \Sym (n, \reals)$ given by 
\begin{equation}\label{eq:covmap}
f(\lambda, \alpha,\bbu)= \lambda \mat{\bbu & *}
\mat{1 & \bb0 \\ \bb0 & \alpha\bbI_{n-1}} \mat{\bbu & *}^\top,
\end{equation}
where $*$ is any basis completion for $\reals^n$, and $\bbI_{n-1}$ is the identity matrix.
The function $f$ essentially builds a covariance matrix from the parameters supplied by $\ccalA$ and $\set{(\lambda, \ccalT_\lambda) \mid \lambda \in \ccalL}.$
{We define the covariance space as
\begin{equation}
\ccalC \triangleq
\set{\bb0} \cup f \left( \ccalA \times \bigcup\nolimits_{\lambda \in \ccalL} 
\left( 
\{ \lambda 
\}\times \ccalT_\lambda \right) \right).
\end{equation}}
%
\noindent The $\bb0$ covariance is an artificial state that we include in $\ccalC $ to denote that no more uncertainty remains in the variable being estimated, i.e., estimation is complete to the user-specified tolerance, defined as $\frac{1}{2} \min \set{\lambda \in \ccalL}.$

The projection operator $\proj_{\ccalC }$ guarantees that the fusion of the current covariance state $\bbSigma_k$ and the new measurement covariance $\bbQ_k$ is a member of $\ccalC.$
In particular, for some $\bbSigma \in \Sym_{++} (n ,\reals),\proj_{\ccalC}$ first computes the principal eigenvalue $\lambda_\text{max}$ and its corresponding normalized eigenvector $\bbu_\text{max}$.
Then, it rounds $\lambda_\text{max}$ to the closest element in $\ccalL \cup \set{0}.$
Call this map $\proj_\ccalL : \Sym(n,\reals) \to \ccalL \cup\set{0}.$
If $\proj_\ccalL(\bbSigma)$ is nonzero,
there will be some $\lambda ' \in \ccalL$ that is closest to $\lambda_\text{max},$
and $\proj_\ccalC$ then finds the element $\bbu' \in \ccalT_{\lambda'}$ that forms the largest magnitude inner product with $\bbu_\text{max}.$
Call this map $\proj_{\ccalT_{\lambda'}} : \Sym(n,\reals) \to \mbS^{n-1},$
where $\mbS^{n-1} = \set{\bbu \in \reals^n \mid \left\langle\bbu,\bbu \right\rangle =1}.$
In particular,
$
\bbu' =
\proj_{\ccalT_{\lambda'}}(\bbSigma) \triangleq \max_{\bbu \in \ccalT_{\lambda'}} \abs{\left\langle\bbu, \bbu_\text{max} \right\rangle }.
$
Finally, $\proj_\ccalC$ computes the ratio of $\lambda_\text{min}(\bbSigma)$ with $\lambda_\text{max}(\bbSigma)$ and finds the closest element $\alpha' \in \ccalA$ to that ratio.
Call this map $\proj_\ccalA : \Sym_{++}(n ,\reals) \to (0,1].$
By this construction, it holds that $(\alpha', \lambda', \bbu') \in \ccalA \times \bigcup\nolimits_{\lambda \in \ccalL} 
\left( \lambda \times \ccalT_\lambda \right),$ 
so that its image of this triplet under $f$ from \eqref{eq:covmap} is guaranteed to be a matrix in $\ccalC$.
In particular, the projection map is given by
$
\proj_\ccalC (\bbSigma) = f 
\left(
\proj_{\ccalL } \left(\bbSigma \right), 
\proj_{\ccalA} \left(\bbSigma\right), 
\proj_{\ccalT_{\proj_{\ccalL } \left(\bbSigma \right)}}(\bbSigma) 
\right),
$

\section{Localization of Multiple Targets}\label{sec:mt}
Assume the collection of all hidden states discussed in Section~\ref{sec:pf} is divided into clusters.
Temporarily let $M$ denote the number of hidden states in a particular cluster.
Denote by $\ccalS_i = \ccalW_i \times \ccalC_i$ the state space local to the $i$-th hidden vector in the cluster.
Each $\ccalS_i$ has 
$ 
N_{\ccalS_i }=
N_{\ccalW_i }
\left(
    1 + N_{\ccalA_i } \sum_{\lambda \in \ccalL_i} N_{\ccalT_\lambda} 
\right)
$
individual states.


Let $\ccalE_i\subset \ccalS_i$ denote the set of initial states that the robot can visit when it first arrives at $\ccalW_i$ to observe $\bbx_i.$
We assume that the first observation of $\bbx_i$ will occur at the boundary of the convex hull of local pose space $\partial \ccalW_i.$
Therefore, we define the \emph{entry points} to the $i$ local state space as
$
\ccalE_i \triangleq \set{(\bbp, \bbSigma_{i,0}) \mid \bbp \in \partial \ccalW_i}.
$
Let us index the states in the set $\ccalE_i$ using integers $ j \in \set{1, \dots, \abs{ \ccalE_i}}$ such that
$
j \overset{\text{1-to-1}}{\longmapsto} \bbs_j \in \ccalE_i .
$
Since in the neighborhood of every hidden state, the robots follow the local optimal policy determined in Section~\ref{sec:st}, a robot that begins observing the $i$-th hidden state at the entry point $\bbs_j \in \ccalE_i$ and has spent $k$ steps at that particular state has a known global location and local covariance matrix.
To keep track, let $\bbPhi_i \colon \ccalS_i \times \ccalU_i \to \ccalS_i$ and 
$\bbmu_i^* \colon \ccalS_i \to \ccalU_i$
denote the $i$-th local transition function from \eqref{eq:joint_dynamics} and optimal policy.
Then, define the $k$-times recursive local optimal transition function as
$\bbPhi_i^{*k} \colon \ccalE_i \to \ccalS_i$, given by
\begin{align}
\bbPhi_i^{*0}(\bbs_j) & = \bbs_j \nonumber\\
\bbPhi_i^{*1}(\bbs_j) & = \bbPhi_i\left(\bbs_j,\bbmu_i^*(\bbs_j )\right) \nonumber \\[-2ex]
&\vdots  \nonumber \\[-2ex]
\bbPhi_i^{*k}(\bbs_j) &=  \bbPhi_i\left(\bbPhi_i^{*k-1}(\bbs_j),\bbmu_i^*\left(\bbPhi_i^{*k-1}\left(\bbs_j \right) \right)\right) .\label{eq:ktrans}
\end{align}
The function $\bbPhi_i^{*k}$ denotes the local state in $\ccalS_i$ in which the robot will land when starting observing $\bbx_i$ at from entry point $\bbs_j \in \ccalE_i$ and after following the local optimal policy $\bbmu_i^*$ for $k$ time steps. 

In the cluster DP, there are $M$ available actions, one for every hidden state in the cluster.
In particular, for a robot observing $\bbx_i$, action $\bbu_i$ is simply to continue observing the same state along the local optimal policy.
Note that, since the optimal policy is stationary, there is always a local optimal action to take.
The following proposition shows that any local optimal trajectory reaches an absorbing state so that the iteration $k \mapsto k+1$ terminates.
\begin{prop}\label{prop:truncate}
$\forall i \in \set{1, \dots, M}$ and $\bbs_j \in \ccalE_i$, there exists $K_i \in \mbN$ such that $\bbPhi_{i}^{*K_i}(\bbs_j) = \bbPhi_{i}^{*k}(\bbs_j) \; \forall k \geq K_i$.
\end{prop}
\begin{proof}
Since each $\ccalS_i$ is finite and the optimal policy $\bbmu^*_i : \ccalS_i \to \ccalU_i$ is fixed, failure to converge is possible only if the optimal policy drives the robot in a cycle.
This is a contradiction to the optimality of $\bbmu^*_i$.
To see why, consider an optimal local trajectory such as $\set{\dots, \bbs, \bbs',\dots,\bbs, \dots} \subset \ccalS_i.$
If $\bbt\bbr \bbSigma < \bbt\bbr \bbSigma'$, then the transition from $\bbs$ to $\bbs'$ contradicts Lemma~\ref{lem:fusion}.
Similarly, if $\bbt\bbr \bbSigma >\bbt\bbr \bbSigma'$, then the transition from $\bbs'$ to $\bbs$ is a contradiction.
If $\bbt\bbr \bbSigma= \bbt\bbr \bbSigma'$, then the robot must have moved in a loop $\set{\dots, \bbs, \bbs',\dots,\bbs, \dots}$ without changing the uncertainty, i.e., energy was consumed for no gain in reward, a contradiction to the optimality of $\bbmu^*_i.$
\end{proof}
\noindent As a result of Proposition~\ref{prop:truncate}, we do not store local optimal trajectories longer than $\max_{i \in \set{1, \dots, M}} K_i$.

We can now define the state-space for the cluster DP.
The cluster state must contain the index of the current hidden vector being estimated $i \in \set{1, \dots, M},$ the entry point $j \in \set{1, \dots, \abs{\ccalE_i}},$ and the amount of time spent observing the current state $k \in \set{0, \dots,  K_i}.$
The cluster state must also contain the visitation history a $\bbv \in \set{0,1}^M$ to prevent rewards from being gained by collecting the same information twice.
The cluster state-space is thus
\begin{align}
\ccalCS \triangleq& \set{0,1}^M \times \set{1, \dots ,M} \times \label{eq:MTstates} \\
&\hspace{0.2cm} \set{1, \dots , \max\nolimits_i  \abs{\ccalE_i}}  \times \set{0, \dots, \max\nolimits_i K_i}, \nonumber 
\end{align}
where the $\max$ functions are needed to account for the largest local state spaces in the cluster.

Let $\ccalCU = \set{1, \dots, M}$ denote the set of control inputs in the cluster DP.
Let also $\PhiC \colon \ccalCS \times \ccalCU \to \ccalCS$ denote the cluster transition function.
When the robot is in state $(\bbv, i,j,k) \in \ccalCS,$ then $\bbu_i \in \ccalCU$ transitions the robot to cluster state $(\bbv, i,j,k+1) \in \ccalCS.$
Action $\bbu_i$ has the same reward with the corresponding transition in the local DP.
The remaining actions $\set{\bbu_\ell \mid \ell \neq i}$ set the visitation history of the $i$-th hidden vector to zero and transition the robot to the $\ell$-th local space, specifically by selecting the closest entry point in in $\ccalE_\ell$ to the robot's current location.
In particular, 
\[
\PhiC \left( (\bbv, i,j,k), \bbu_\ell \right) =
\begin{cases}
(\bbv, i,j,k+1) &\text{if } \ell=i \\
(\bbv', \ell ,j',0) &\text{if } \ell \neq i
\end{cases},
\]
where 
$\bbv_i'=0$ and 
\begin{equation}\label{eq:nextentry}
j' = 
\argmin\nolimits_{j'' } \norm{ \bbs_{j''} - \Phi_i^{*k} (\bbs_j) } \mid \bbs_{j''}\in \ccalE_\ell,\bbs_j \in \ccalE_i
.
\end{equation}

The reward for continuing the local optimal policy is the same as the reward in the local DP, and the reward for transitioning to a new hidden state in the cluster is the negative distance to the next entry point $j'$ defined in \eqref{eq:nextentry}. 
The only other difference with the local reward is that the robot cannot gain positive reward for taking action $\bbu_i$ when at any state $(\bbv, i, \cdot, \cdot)$ such that $\bbv_i=0.$
In particular, define the reward function in the cluster DP $R_C \colon \ccalCS \times \ccalCU \to \reals$ as
\agn*{
R_C &\left( (\bbv, i,j,k), \bbu_\ell \right) =
\\
&\begin{cases}
R \left( \bbPhi^{*k}_i \left(\bbs_j \right) ,\bbmu_i^* \left(\bbPhi^{*k}_i \left(\bbs_j \right) \right) \right) &\text{if } \bbv_i = 1, \ell=i \\
-\rho \bbpsi(\bbu_\ell) &\text{else}
\end{cases}.
}

{
\begin{rem}\label{r:tsp}
One could directly apply the method developed in Section~\ref{sec:st} to the task of sensing multiple targets at once.
The state-space in this case would be of the form
$
\pair{ \cup_{i =1}^M \ccalW_i} \times \ccalC_1 \times \cdots \times \ccalC_M,
$
which has 
\begin{equation}\label{eq:sizebrute}
\left(\sum_{i=1}^M N_{\ccalW_i}\right) \prod_{i=1}^M \underbrace{\left( 1 + N_{\ccalA_i } \sum_{\lambda \in \ccalL_i} N_{\ccalT_\lambda} 
\right)}_{N_{\ccalC_i}}
\end{equation}
states, an intractably large number of states for any sufficiently rich belief space $\ccalC_i.$
This is the reason why existing approaches have typically employ online policy search or point-based solvers.
Our proposed hierarchical approach solves a simple local DP once to find a local optimal policy for every state, mitigating the exponential complexity to the cluster DP, where the state space has size $O(2^M)$.
Solving the cluster DP has similar in complexity to TSP solvers although it solves a substantially more complicated problem.
\end{rem}
}

\section{Optimal Planning and Resource Allocation for Multiple Robots}\label{sec:many}Given the single-robot optimal control policies developed in Sections \ref{sec:st} and \ref{sec:mt}, in this section we develop a distributed framework to synthesize them in a multi-robot system that can efficiently estimate large groups of hidden vectors. 
%
%
To develop the proposed framework, let $\ccalT=\{1,\dots,M\}$ denote the index set of all available hidden states, and for every hidden state $t\in\ccalT$ define the discrete set $\ccalZ^t \triangleq \big\{ \bbz\in\reals^d \mid (\bbz,\theta)\in\ccalW_t \big\}$ containing those states of the local workspace $\ccalW_t$ that correspond to robot positions in $d=2$ or $3$ dimensions only (excluding other configuration information contained in $\theta$). 
Let $\ccalZ=\cup_{t\in\ccalT}\ccalZ^t$ denote the set of all possible robot positions. 
Consider further a partition $\set{ \ccalT_p}_{p=1}^P$ of the hidden state set $\ccalT$ into $P\leq M$ clusters so that two hidden states belong in the same cluster if they are sufficiently close to each other based on the initial belief of their locations. 
Let $\ccalZ_p \triangleq \big\{ \bbz\in\ccalZ^t \mid t\in\ccalT_p \big\}$ denote the set of all robot positions in cluster $\ccalT_p$. Using the methods developed in Sections~\ref{sec:st} and \ref{sec:mt}, we can determine an optimal sensing policy for every cluster of hidden states $\ccalT_p$, that is a function from the state-space of robot positions and hidden state uncertainties to the set of robot actions. 
Then, given an entry point from where a robot can begin sensing cluster $\ccalT_p$ this policy can be combined with the system dynamics \eqref{eq:robot_dynamics} and \eqref{eq:covariance_dynamics} to generate an optimal robot trajectory within that cluster. 
In particular, we define the set of \emph{entry points} of cluster $\ccalT_p$ to be the set of states in $\ccalZ_p$ lying on the boundary of its convex hull,  denoted by $\partial\ccalZ_p$. 
We also define an optimal trajectory in cluster $\ccalT_p$ starting from a point $\bbp_{j_p}\in\partial\ccalZ_p$ by $\bbxi_{p,j_p}:[0,L_{p,j_p}]\to\ccalZ_p$, where $L_{p,j_p}>0$ denotes the total length of that trajectory from the entry point until the cluster is completed and $j_p$ is an index of the $j$-th entry point in $\partial\ccalZ_p$.

Consider now $N$ mobile robots and let $\set{ \bbz_i^k}_{i=1}^N \subset \reals^d$ denote their locations at time $k\in \mbN$.
While sensing cluster $\ccalT_p$ a robot moves along the optimal trajectory $\bbxi_{p,j_p}$.
When the exploration of $\ccalT_p$ has been completed, the robot needs to transition to a different cluster.
For this, it needs to identify an entry point in the new cluster and travel to that point.
Therefore, the set of all possible paths that a robot can follow while sensing different clusters of hidden states can be represented by graph $\ccalG=(\ccalV,\ccalE)$, where $\ccalV=\big(\cup_{p=1}^P\partial\ccalZ_p\big)$ denotes the set of vertices (embedded in $\reals^d$) containing the depot and all entry points of the $P$ hidden state clusters and $\ccalE\subset \ccalV\times\ccalV$ is the set of directed edges so that 
for any entry points $\bbp_{j_p}\in\partial\ccalZ_p$ and $\bbp_{j_q}\in\partial\ccalZ_q$, the edge $(\bbp_{j_p},\bbp_{j_q})\in\ccalE$ if $p\neq q$ and $\bbp_{j_q}=\textrm{argmin}_{\bbp\in\partial \ccalZ_q} \|\bbxi_{p,j_p}(L_{p,j_p})-\bbp\|$.
In other words, the directed edges in $\ccalE$ connect
every entry point in $\ccalV$ to the closest entry points of different clusters.
With every edge $(\bbp_{j_p},\bbp_{j_q})\in\ccalE$ in $\ccalG$ we associate a distance that the robot needs to travel in order to get to entry point $\bbp_{j_q}\in\partial\ccalZ_q$ having started form point $\bbp_{j_p}\in\partial\ccalZ_p$.
This distance consists of the distance $L_{p,j_p}$ that the robot needs to travel within cluster $\ccalT_p$ plus the distance $\|\bbxi_{p,j_p}(L_{p,j_p})-\bbp_{j_q}\|$ that the robot needs to travel to reach entry point $\bbp_{j_q}\in\partial\ccalZ_q$ once it has completed cluster $\ccalT_p$, i.e.,
$
d(\bbp_{j_p},\bbp_{j_q}) = L_{p,j_p} + \|\bbxi_{p,j_p}(L_{p,j_p})-\bbp_{j_q}\|.
$
We assume that between clusters, robots travel on straight-line paths.
An illustration of the graph $\ccalG$ of possible motion paths for the running example of sparse landmark localization in two dimensions, containing the cluster entry points and trajectories to two possible next clusters, is shown in Figure~\ref{fig:clusters}.
{To avoid having two robots select the same cluster and fail to resolve this conflict, we assume that the communication range of the robots is larger than the largest diameter of any cluster.}

\begin{figure}
  \centering
    \includegraphics[width=8cm]{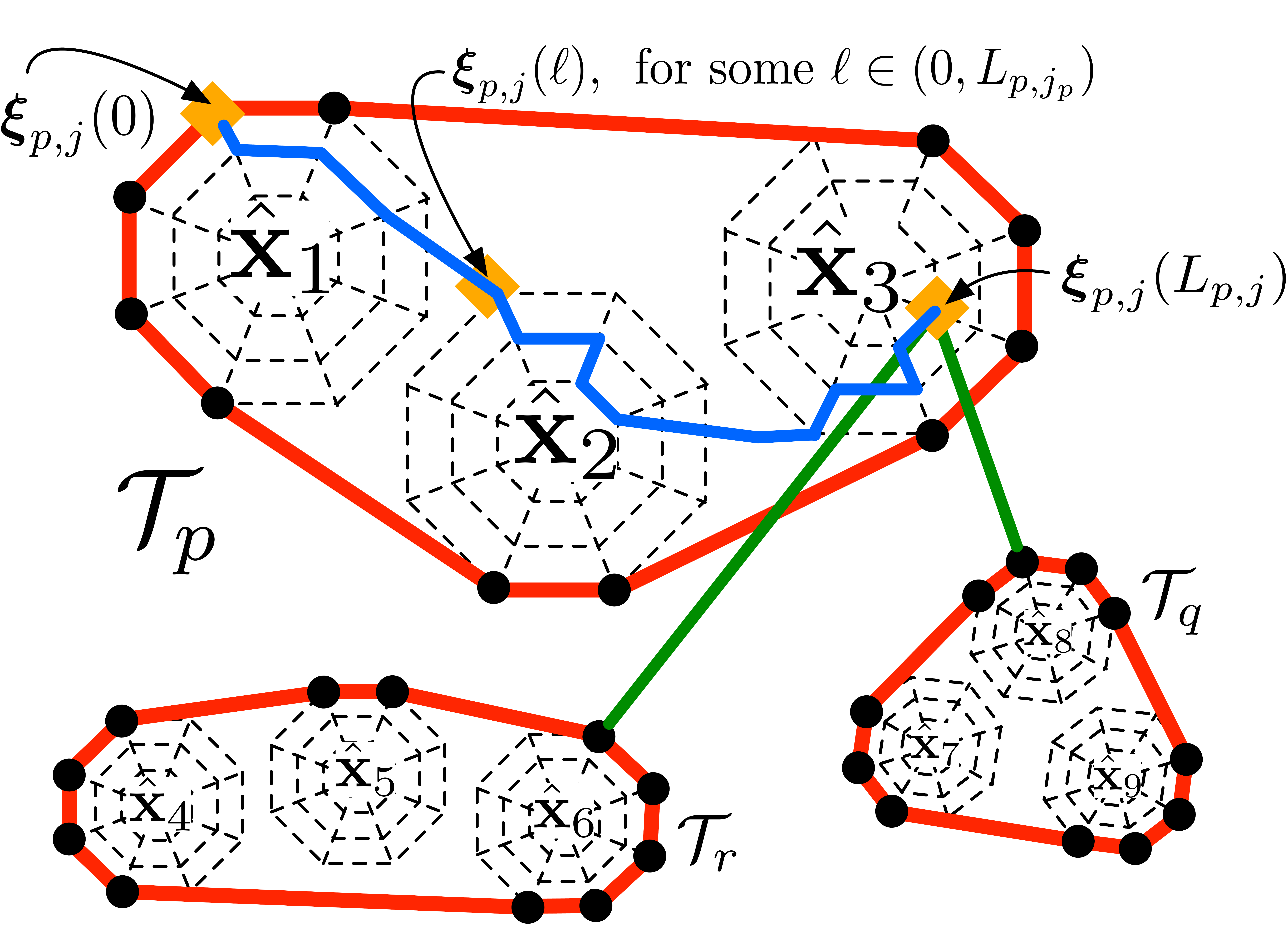}
  \caption{
An illustration of two potential sensing paths for a single robot for $d=2$ in a landmark localization task.
The blue segment is the range of $\bbxi_{p,j}$, where the top left entry point is arbitrarily assigned the index $j$.
The robot's progress at three instants, corresponding to $\ell=0$, $\ell=L_{p,j}$, and some intermediate value $\ell \in (0,L_{p,j})$, are shown as orange diamonds.
The choices of routes to two next clusters $\ccalT_q$ and $\ccalT_r$ are shown as green segments emanating from $\bbxi_{p,j}(L_{p,j})$.
}\label{fig:clusters}
\vspace{-0.75cm}
\end{figure}

In what follows, we develop a distributed framework to allow the team of $N$ robots to dynamically allocate the $P$ clusters amongst themselves, as they plan trajectories on the graph $\ccalG$ to visit their assigned hidden state clusters.
Specifically, we assume that at every time $k$, every robot $i$ can be in one of three modes $m_i^k\in\{\textrm{`busy', `transit', `done'}\}$. 
We say that a robot is `busy' if it is currently sensing a cluster, it is in `transit' if it is traveling between clusters, and it is `done' if it has completed sensing its last cluster. 
While in mode `busy' robot $i$ moves along the trajectory $\bbxi_{p,j_p}$ according to
\begin{align}\label{eq:busynav}
\bbz_i^{k+1}=\busy (\ell^k, p, j_p) \triangleq \bbxi_{p,j_p}(\ell^k+\delta\ell),
\end{align}
where $\ell^k=\bbxi_{p,j_p}^{-1}(\bbz_i^k)\in[0,L_{p,j_p}]$ denotes the distance that robot $i$ has already travelled in cluster $\ccalT_p$ and $\delta\ell>0$ is a small, user-defined, positive distance increment that the robot travels between times $k$ and $k+1$. While in mode `transit', robot $i$ moves according to
\begin{align}\label{eq:transitnav}
\bbz_i^{k+1}=\transit (\bbz_i^k, q) \triangleq \bbz_i^k + \delta\ell \frac{  \bbp_{j_q} - \bbz_i^k }{ \norm{ \bbp_{j_q} - \bbz_i^k }},
\end{align}
where $q$ denotes the index of the cluster that the robot is traveling to, $\bbp_{j_q}\in\partial\ccalZ_q$ is the selected closest entry point in that cluster defined as
$
\bbp_{j_q}=\textrm{argmin}_{\bbp\in\partial \ccalZ_q} \|\bbz_i^k-\bbp\|,
$
and $\delta\ell>0$ is defined as in \eqref{eq:busynav}. Assuming that robot $i$ is in transit to cluster $\ccalT_q$ after having completed cluster $\ccalT_p$, the line segment defined by the end points $\bbz_i^k$ and $\bbp_{j_q}$ is completely contained in the line segment defined by the points $\bbxi_{p,j_p}(L_{p,j_p})$ and $\bbp_{j_q}$. 
Therefore, \eqref{eq:transitnav} drives the robot along the path corresponding to the edge $(\bbp_{j_p},\bbp_{j_q})\in\ccalE$. 
However, the controller \eqref{eq:transitnav} also allows robot $i$ to diverge from the predefined paths in the graph $\ccalG$ by selecting an alternative closest cluster while in transit mode, for reasons that we discuss in Section \ref{sec:auctions}, e.g., if another robot places a higher bid. 
In this case, the motion of robot $i$ temporarily leaves the predefined motion paths in $\ccalG$ and it re-enters $\ccalG$ once it has reached cluster $\ccalT_q$.
In the remainder of this section, we will refer to the current cluster being sensed by the $i$-th robot as $c_{i,\text{curr}}$ and the next cluster to be sensed by that robot as $c_{i,\text{next}}$. 
Then, our goal is to find a set of $N$ distinct paths in $\ccalG$ whose union visits every cluster exactly once and has a minimum combined travel cost.  
We achieve this goal by a distributed auction mechanism that we discuss next.

\subsection{Distributed Auction Mechanism}\label{sec:auctions}


In this section we propose a distributed auction method to dynamically and sequentially allocate clusters to robots as they move to localize the whole scene.
{
A proof of the convergence of distributed auctions can be found in \cite{zavlanos2008distributed}.
}

%
%
Specifically, let $s=1,2,\dots$ denote a sequence of time instants when the robots communicate with each other, that is in general different from the times $k$ when the robots move, and let $\ccalN_i^s \triangleq \big\{j \mid \|\bbz_i^s-\bbz_j^s\|<\Delta\big\}$ denote the set of neighbors of robot $i$ at time $s$, where $\Delta>0$ denotes a given communication range. Moreover, assume that every robot $i$ carries two lists: the list of `free' clusters $\ccalI_{i,f}^s$ and the list of `taken' clusters $\ccalI_{i,t}^s$, so that $\ccalI_{i,f}^s \cup \ccalI_{i,t}^s = \set{1, \dots, P}$ and $\ccalI_{i,f}^s \cap \ccalI_{i,t}^s = \emptyset$ for all time $s$.
Initially, $\ccalI_{i,f}^0 = \set{1, \dots, P}$ and $\ccalI_{i,t}^0 = \emptyset$. The list of `free' clusters contains clusters that are available to robot $i$, meaning that robot $i$ can select from those clusters a cluster to visit next. On the other hand, the list of `taken' clusters contains clusters that have been selected by other robots and are, therefore, not available to robot $i$. During operation, robot $i$ coordinates with its neighbors $j \in \ccalN_i^s$ to update its list of `taken' and `free' clusters by $\ccalI_{i,t}^{s+1}=\cup_{j \in \ccalN_i^s } \ccalI_{j,t}^s$ and $\ccalI_{i,f}^{s+1}=\{1,\dots,P\}\setminus \ccalI_{i,t}^{s+1}$, respectively. In other words, with every communication round, robot $i$ removes {\blue from} its list of free clusters those clusters that are considered taken by other robots.

Given the list of `free' clusters $\ccalI_{i,f}^s$ at time $s$, robot $i$ can select any cluster from that list to be the next cluster $c_{i,\text{next}}$ to visit. To minimize the total distance travelled by the robots, we propose a greedy approach where robots select a cluster that is the closest to their current location. In particular, we define
\begin{align}\label{eq:cnext}
&c_{i,\text{next}}^{s+1} =   \! \begin{cases}
\argmin_{c \in \ccalI_{i,f}^s}  d_{p,j_p}^\text{busy}(\ell^s, c) &\text{if }m_i^s = \text{`busy'} \\
 \argmin_{c \in \ccalI_{i,f}^s} d^\text{trans} (\bbz_i^s,c)  &\text{if }m_i^s = \text{`transit'}
\end{cases}
\\
&
\label{eq:busy}
d_{p,j_p}^{\text{busy}}(\ell^s, c) \triangleq (L_{p,j_p} \! - \! \ell^s) +  \! \!    \min_{\bbp\in\partial\ccalZ_c} \|\bbxi_{p,j_p}(L_{p,j_p}) \! - \! \bbp\|\\
&\text{and }
\label{eq:trans}
d^{\text{trans}}(\bbz_i^s, c) \triangleq \min_{\bbp\in\partial\ccalZ_c} \|\bbz_i^s-\bbp\|,
\end{align} 
In \eqref{eq:busy} and \eqref{eq:trans}, $d_{p,j_p}^{\text{busy}}(\ell^s, c)$ and $d^{\text{trans}}(\bbz_i^s, c)$ are the distances that robot $i$ needs to travel in order to reach a new cluster $c$ from its current location $\bbz_i^s$ while in modes `busy' and `transit', respectively, and $\ell^s=\bbxi_{p,j_p}^{-1}(\bbz_i^s)\in[0,L_{p,j_p}]$, as in \eqref{eq:busynav}. When robot $i$ selects a new cluster $c_{i,\text{next}}^{s}$, then it also updates its lists of `free' and `taken' clusters by removing $c_{i,\text{next}}^{s}$ form $\ccalI_{i,f}^s$ and adding it to $\ccalI_{i,t}^s$.

Every time $c_{i,\text{next}}$ is updated, robot $i$ also places a bid that indicates how important the selection of the new cluster is. The bids are inversely proportional to the distance  robot $i$ needs to travel to reach the new cluster, so that nearby clusters have higher value. Specifically, bids are placed according to
\begin{align}\label{eq:bids}
b_i^{s+1} &= \begin{cases}
\max_{c \in \ccalI_{i,f}^s}  \left(1+d_{p,j_p}^\text{busy}(\ell^s, c) \right)^{-1} &\text{if }m_i^s = \text{`busy'}, \\
\max_{c \in \ccalI_{i,f}^s} \Big(1+d^\text{trans} (\bbz_i^s,c) \Big)^{-1} &\text{if }m_i^s = \text{`transit'}
\end{cases}.
\end{align}
If at some point in time there exist neighbors $j\in\ccalN_i^s$ of robot $i$ so that $c_{j,\text{next}}^s = c_{i,\text{next}}^s$, then these robots set up a local auction and compare their bids to resolve the underlying conflict. If $b_i^s>b_j^s$ for all $j\in\ccalN_i^s$ for which $c_{j,\text{next}}^s = c_{i,\text{next}}^s$, then robot $i$ wins the auction and maintains the same next cluster and bid, i.e., $c_{i,\text{next}}^{s+1} = c_{i,\text{next}}^s$ and $b_i^{s+1}=b_i^s$. The robots that lose the auction update their set of `free' clusters by removing cluster $c_{j,\text{next}}^s$, i.e., $\ccalI_{j,f}^s=\ccalI_{j,f}^s\setminus \{c_{j,\text{next}}^s\}$, and select a new next cluster and bid according to \eqref{eq:cnext} and \eqref{eq:bids}. If $\ccalI_{i,f}^s=\emptyset$, i.e., if there are no other available clusters for robot $i$, we set $c_{i,\text{next}}^{s+1} = \textrm{`depot'}$, effectively controlling the robot to return to a depot after it has completed its current (final) task. 
%
%
%

\begin{figure*}[t]
\begin{center}
\includegraphics[width=11cm]{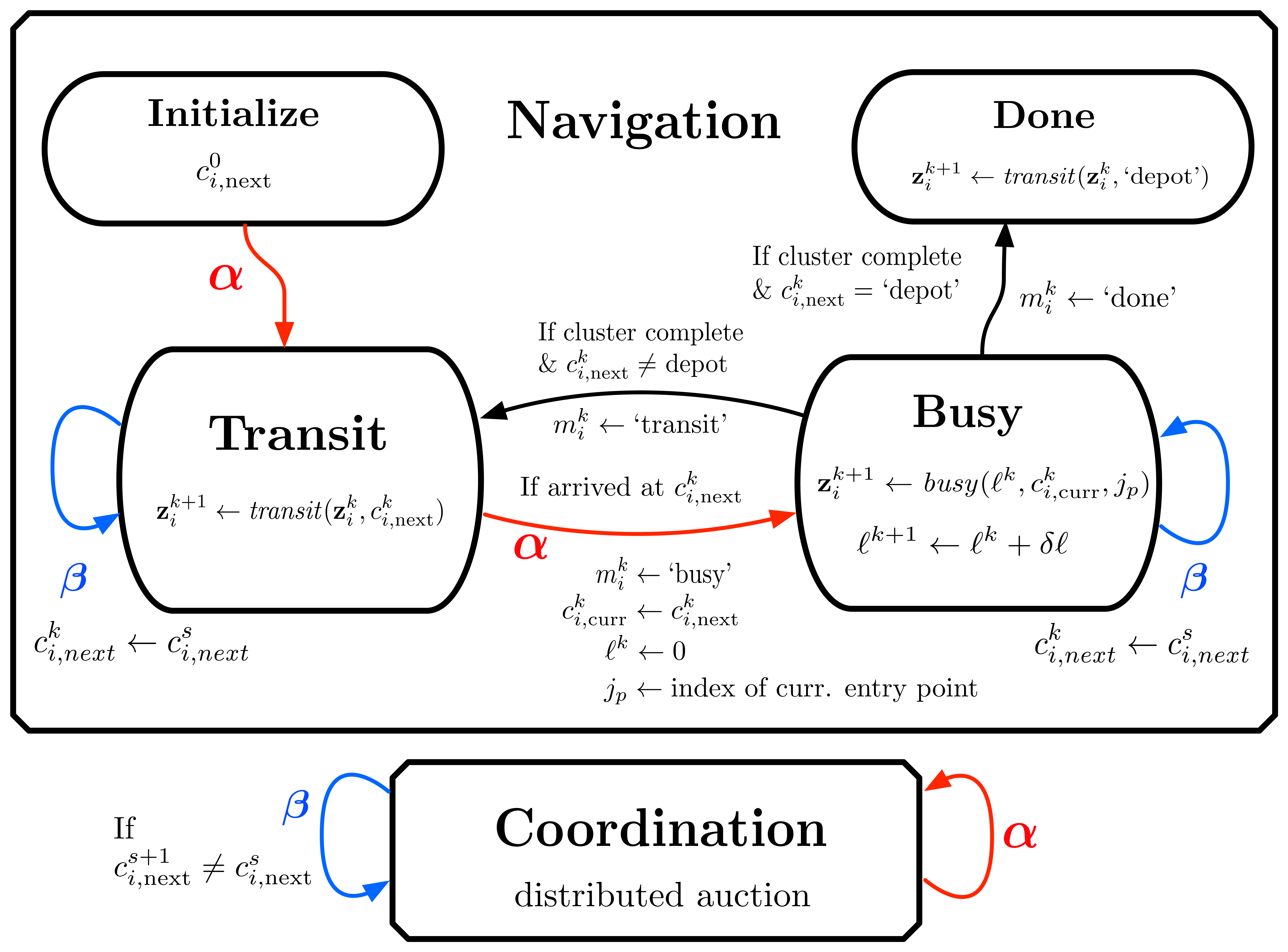}
\caption{
A diagram representing the controller for the $i$-th robot.
The two blocks, Navigation and Coordination, synchronize whenever the updates denoted by $\alpha$ and $\beta$ are triggered.
Functions \emph{transit} and \emph{busy} are defined in \eqref{eq:transitnav} and \eqref{eq:busynav}.
}
\label{fig:automaton}
\end{center}
\vspace{-0.5cm}
\end{figure*}

Fig~\ref{fig:automaton} illustrates the integrated, hybrid, controller.
The labels $\alpha$ and $\beta$ mark events that need to be synchronized across the navigation and coordination control blocks.
In particular, transitions labeled by the letter $\alpha$ are triggered by the navigation block and generate synchronous transitions in the coordination block aimed to produce new bids or update the next cluster that the robot needs to visit. 
Similarly, transitions labeled by the letter $\beta$ are triggered by the coordination block when new bids are computed or the next cluster that the robot needs to visit is updated and they generate synchronous transitions in the navigation block aimed to guide the robot its new assigned cluster. 
Note that while the $k$ and $s$ time indices used for the navigation and coordination blocks, respectively, can in general be different, the transitions labeled by the letters $\alpha$ and $\beta$ can generate transitions in these blocks that can be off-clock. 
For example, a transition at time $k$ in the navigation block labeled by $\alpha$ will generate a transition in the coordination block at a time instant $k\neq s$.

\section{Numerical Simulations}\label{sec:sim}

In this section, we present simulations of the proposed distributed state estimation algorithm.
In our simulations, we focus on the problem of sparse landmark localization.  
As a sensor model, we use a stereo camera.
We refer the reader to \cite{freundlich13icra,freundlich13cdc} for a discussion of the covariance function $\bbQ(\hbbx_i,\bbp)$ for stereo vision.
We assume that each camera in the simulated rig has $1024 \times 1024$ resolution and a 70$^\circ$ field of view.
The characteristic length in stereo vision is the baseline.
Therefore, in these simulations, \emph{all units are measured in stereo baselines unless otherwise stated.}
For a mobile stereo rig, the baseline could be as large as 1 meter or as small as a few centimeters.

To simplify the exposition, we consider fully actuated kinematic robots.
The local pose spaces and prior error covariances are identical for each landmark, i.e., $\ccalW_i = \ccalW_j$ and $\bbSigma_{i,0} = \bbSigma_{j,0}=\beta \bbI_3$ for all $i,j \in \set{1, \dots,M}$, where we recall that the largest eigenvalue in the dimension of the state-space that represents the possible principal eigenvalues is denoted as $\beta.$

\subsection{Active sensing of a single target}\label{sec:st_results}

The local pose space $\ccalW$ for the single target case is made up of concentric spherical shells with randomly distributed viewpoints on each.
There were 177 total views in $\ccalW$ at six equally spaced radii between 20 and 40 units from the landmark.
To discretize the covariance space $\ccalC$, we set $1 = \max \set{\lambda \in \ccalL} \triangleq \beta $.
In total, the number of possible principal eigenvalues was $N_\ccalL = 6$.
We also use $N_\ccalA = 3$ condition numbers. 
We set the number of principal eigenvectors for covariances with $\beta$ as the principal eigenvalue to $N_{\ccalT_\lambda} = 98.$
This means $\ccalT_\beta$ has 98 unit vectors.
The number of samples at other principal eigenvectors $\set{\lambda \in \ccalL \mid \lambda \neq \beta}$ was $
\ceil{
\frac{\lambda}{\beta}
N_{\ccalT_\beta} 
}.$
For the logspace discretizations, we set $\kappa_\ccalL=9,$ and $\kappa_\ccalA=3.$
Stereo vision is used as the sensing model; see Appendix~\ref{app:noise}.
The discount factor $\gamma$ was 0.99.
The value of the uncertainty reduction gain $\rho$ is very important, since it controls the tradeoff between traveling cost and uncertainty reduction due to more images.
Fig.~\ref{fig:error} shows the sensitivity of the total reward gained and the two opposing objectives that comprise it as they depend on $\rho$.
Fig.~\ref{fig:rho999} shows example an optimal trajectory in both pose and covariance space with $\rho=0.999.$

\begin{figure}[t]
\centering
\includegraphics[width=8cm]{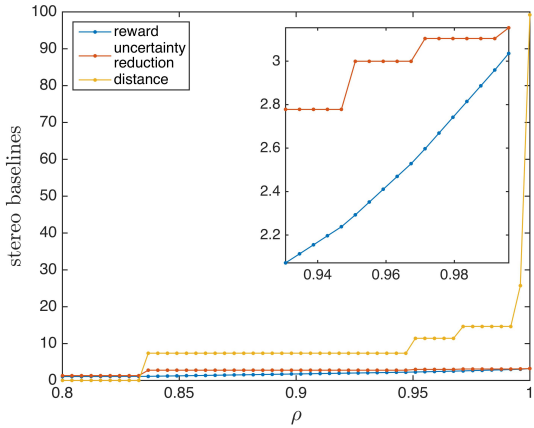} 
\caption{Showing the effect of the uncertainty gain parameter $\rho$ on the total reward, uncertainty reduction, and total distance traveled by the sensor in the single landmark localization simulations.
Lines are drawn to guide the eye
}
\label{fig:error}
\vspace{-0.0cm}
\end{figure}

\begin{figure}[t]
\centering
\includegraphics[width=2in]{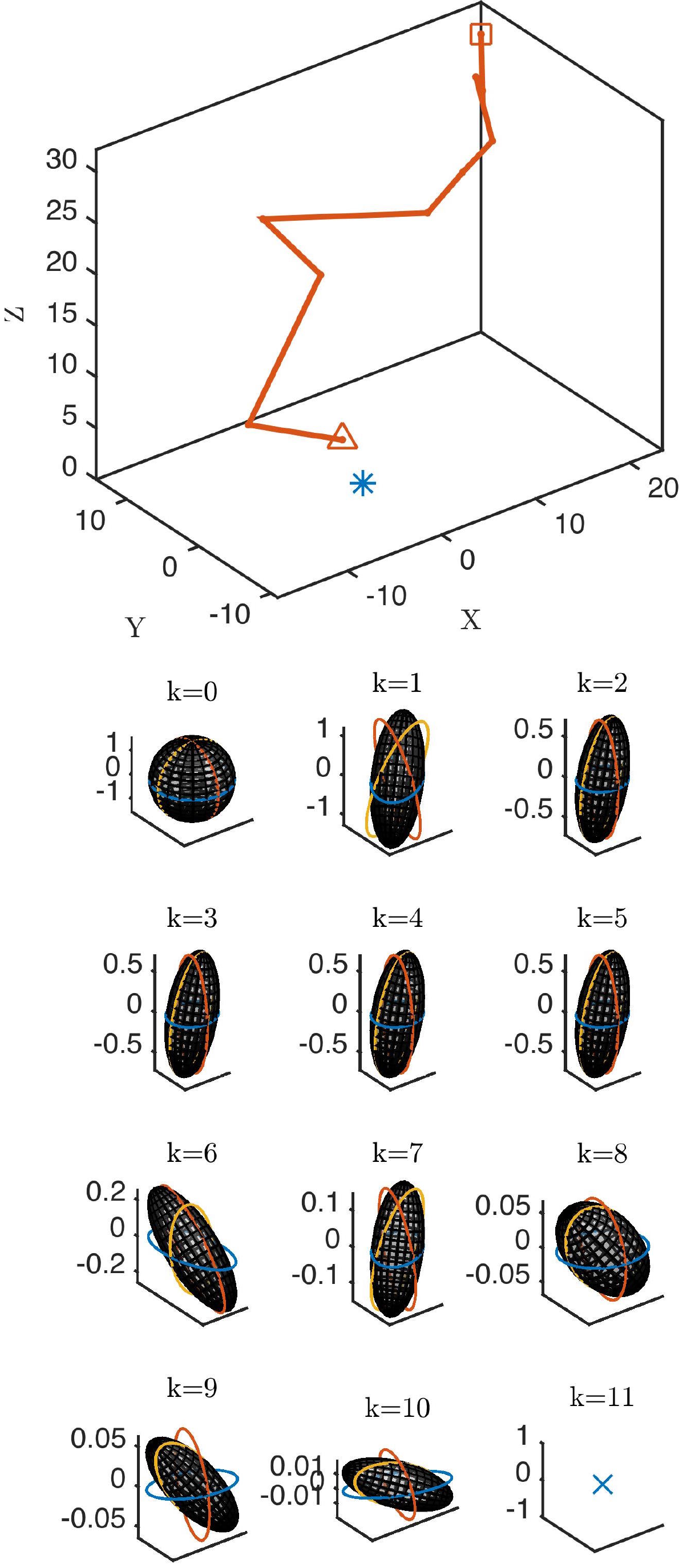} 
\caption{Showing an optimal trajectory through both pose (top panel) and covariance (bottom six panels) space for $\rho = 0.999.$
The indices in the top panel correspond to the time indices in the bottom panel.
Covariance states are shown as 50\% confidence regions.
Robot starts at $\square$ and ends at $\triangle.$
All units in stereo baselines and axis directions equally scaled
}
\label{fig:rho999}
\vspace{-0.0cm}
\end{figure}

In these simulations, we approximate error distributions as state-dependent white noise using the equations given in Appendix~\ref{app:noise}.
White noise is uncorrelated in time, thus, under standard KF assumptions, two observations from the same vantage point will reduce variance by Lemma~\ref{lem:fusion}.
This is why, in the figures, there are more intermediate covariance ellipsoids than steps in the robot trajectories; it is sometimes optimal to remain stationary and take multiple observations.
In real scenarios, noise is seldom well-approximated as white, and duplicate observations will exacerbate bias.
We account for such hidden biases in these simulations by using a real model of stereo vision that is subject to the nonlinear effects of quantization on the image plane; see Appendix~\ref{app:noise}.
Interestingly, the presence of such noise reveals a secondary benefit of sensor mobility:
mobile sensors avoid by default duplicate observations that would plague a static sensor.
Although our controller does not explicitly encourage motion for this reason, a basic heuristic adjustment that does would be simple to implement by, e.g., removing the action ``remain stationary.''

\begin{table}
\centering
\caption{Comparison to Heuristics for Single Landmark}
\label{tab:heur}
\begin{tabular}{| c | c c c c c |}
\hline 
Method       & optimal  & closer & static & circle & random \\ \hline
Reward (m)& 3.15    & 2.75	& 1.31   & 1.66   & 2.74 \\ \hline
\end{tabular}
\end{table}

Table~\ref{tab:heur} compares the total reward gained by a robot starting from the same initial condition under the optimal policy with $\rho=0.999$ found by our (optimal) local controller with some heuristic policies.
The heuristics we chose for comparison were a policy that drives closer to the hidden state (closer), one that circles around the source (circle), one that remains still (static), and one that acts randomly (random).
The fact that that `random' performs about as well as `closer' suggests that a sophisticated controller is needed in order to beat `random.'
Our `optimal' policy obtains 115\% of the reward of the `closer' and `random' policies.
Also note that the `static' and `circle' policies perform very poorly in comparison to the others.
This is because they stay at a constant depth with respect to the landmark.
In triangulation with stereo vision, there is significant bias in the viewing direction, causing poor localization performance given multiple observations at constant depth.

\subsection{Active sensing of Target Clusters}\label{sec:mt_results}

\begin{figure}
\centering
\begin{tabular}{c}
\includegraphics[width=7cm]{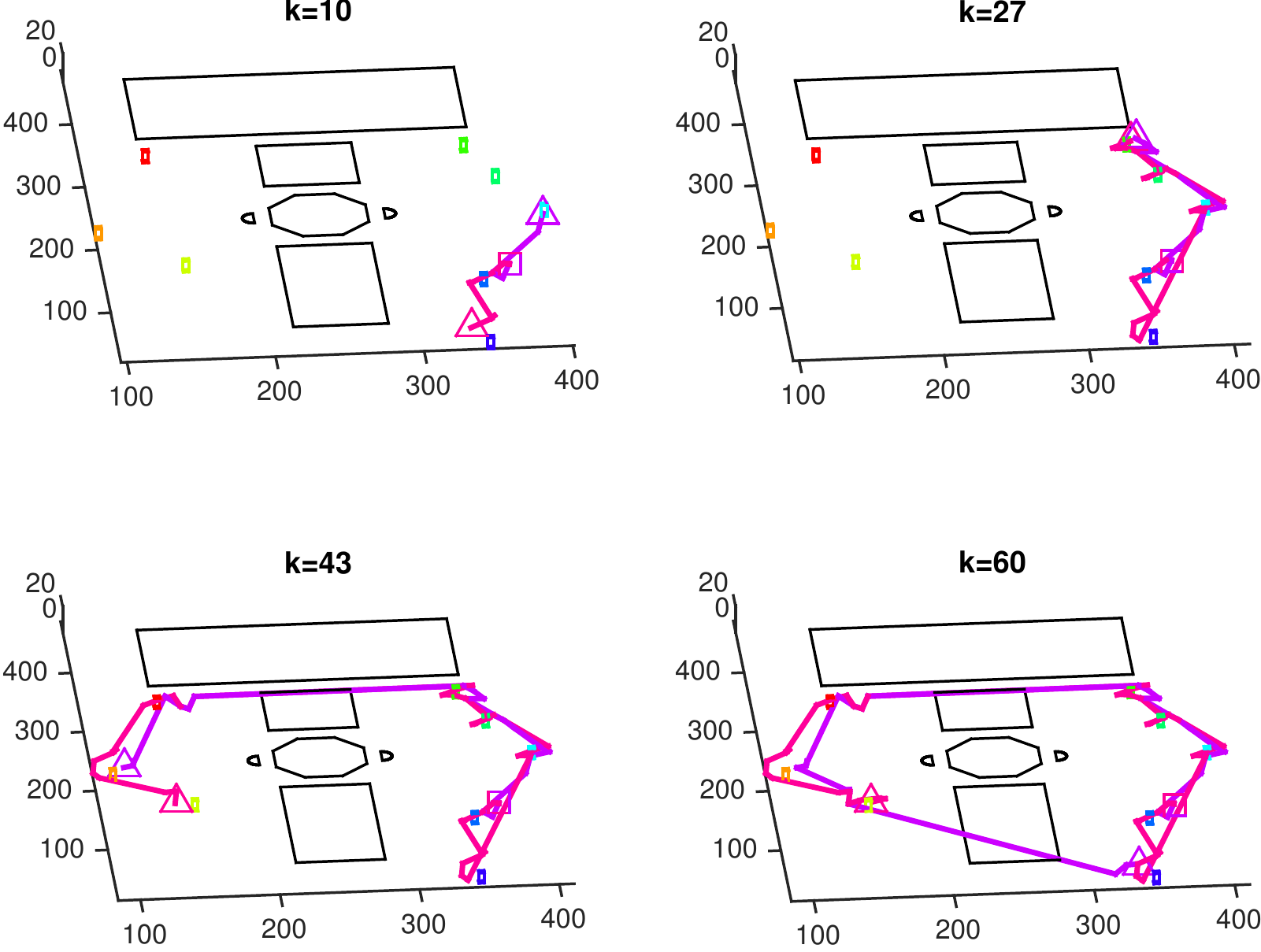}
\\
\includegraphics[width=7cm]{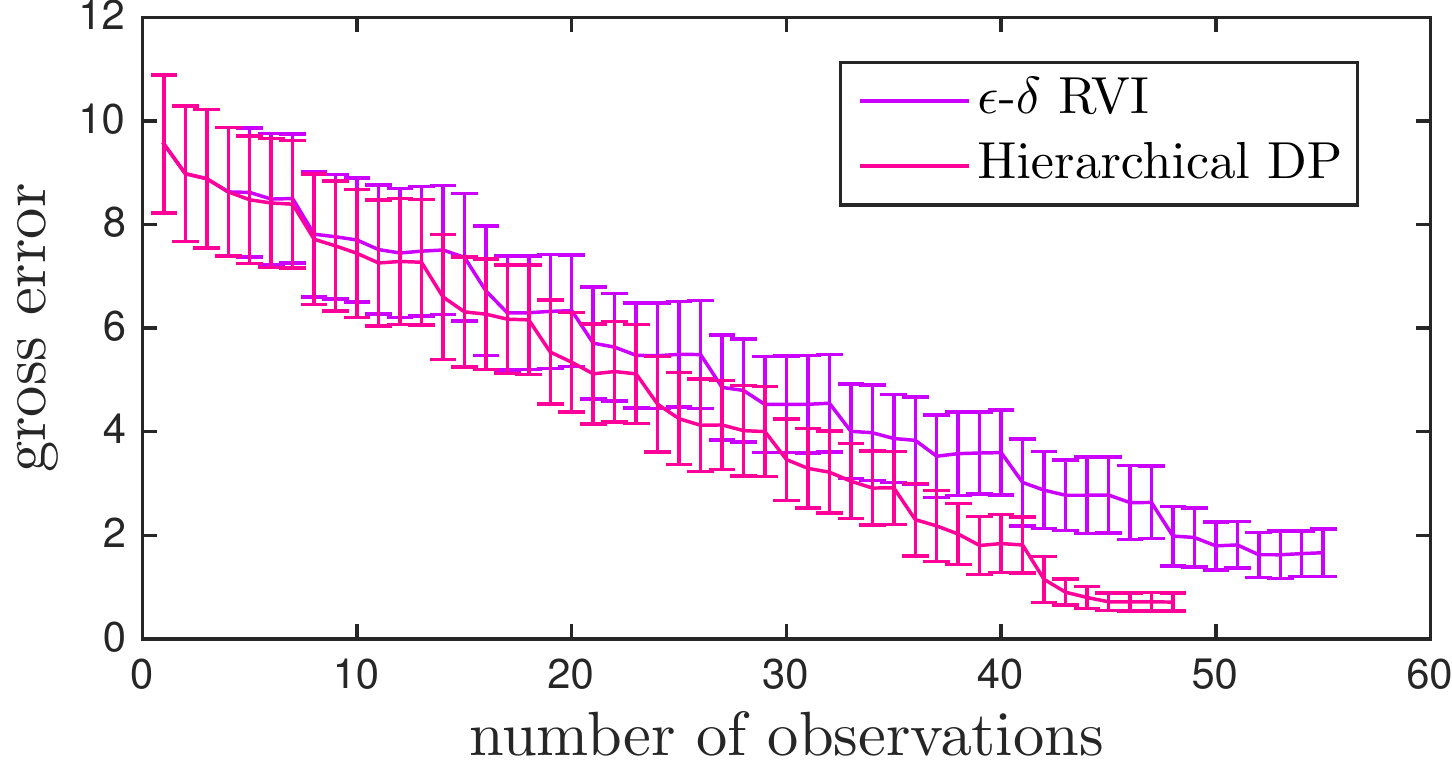}
\end{tabular}
\caption{
Top panel:
Showing trajectories generated by our hierarchical controller and by the method in \cite{atanasov14}.
The trajectories start at the $\square$ and end at the $\triangle$.
The number {\bf k} indicates the number of observations taken.
Bottom Panel: plotting the sum of errors in all landmarks versus the number of observations required found over 100 simulations of observations taken along the trajectories in the top panel.
The midlines represent the empirical mean and error bars represent the standard deviation of the sum of error vectors over the 100 simulations.
All units in meters
}
\label{fig:cluster}
\vspace{-0.0cm}
\end{figure}

We present simulations of the proposed distributed estimation method for a single robot observing a cluster of sparse landmarks.
For the landmarks, we model a famous group of sculptures called the Queens of France and Famous Women, which can be found in the Luxembourg Garden, Paris, France.
The Luxembourg Garden is actually home to hundreds of sculptures, and the Queens of France and Famous Women is a cluster that surrounds a large pool (octagon in Fig.~\ref{fig:cluster}) that is adjacent to the Luxembourg Palace (large rectangle in Fig.~\ref{fig:cluster}).
We chose a subset of eight queens as individual sparse landmarks to comprise a cluster of statues.
The task of the robot in this scenario is to exactly localize each statue to create a precise spatial map of the sculpture garden.
For these simulations, we used a discount factor and gain of 0.999.
We use the same hemispherical pose state-spaces as in Section~\ref{sec:st_results} for each local DP.

We also compare the trajectories generated by our cluster DP with a state-of-the-art algorithm \cite{atanasov14}.
The method in \cite{atanasov14}, called $\epsilon$-$\delta$ Reduced Value Iteration ($\edRVI$), grows a tree in belief space that computes an optimal sensor trajectory, i.e., it applies the most widely used approach to the same problem that we solve under similar assumptions.
In our implementation of $\edRVI$, planning horizons greater than 12 caused our simulation, run on Macbook Air\texttrademark with 4 GB of RAM, to run out of memory.
Because of the sparsity of the scene, 12 step lookahead was sometimes not enough to plan future landmarks to visit.
Therefore, if the robot following $\edRVI$ finishes observing a landmark to the threshold, it greedily selects a new landmark.
%
The top panel in Fig.~\ref{fig:cluster} displays four snapshots of two different trajectories: one produced by the cluster DP optimal policy and one produced using $\edRVI$.
%
The bottom panel of Fig.~\ref{fig:cluster} shows the result of the KF output of 100 simulations of observations taken along the trajectories in the top panel.
The empirical standard deviations for each error vector are also drawn on the figure.
Both methods perform similarly, with our method requiring slightly less observations.
The important distinction, however, is that for every new initial condition, any tree-based planner, including $\edRVI$, needs to be run again, whereas we can reuse our optimal policy for any initial condition.
This is important in the next section, as it allows allocation of clusters dynamically in the  multi-robot team from a variety of initial conditions along the cluster boundaries.
In other words, our method allows us to compute an \emph{optimal policy}, rather than a single \emph{optimal trajectory}.

\subsection{Multiple Robots and Multiple Clusters}

\begin{figure}[t]
\centering
\includegraphics[width=7cm]{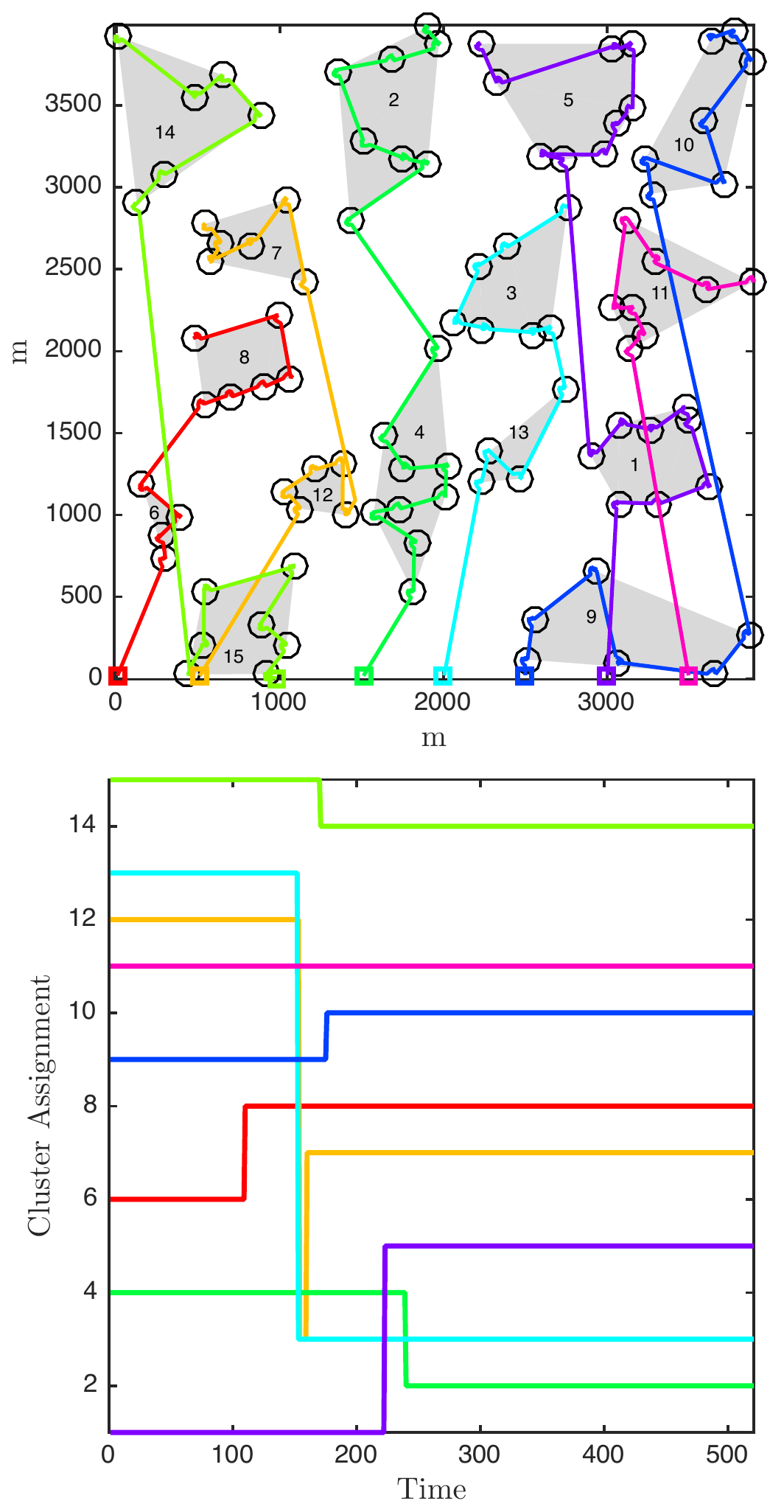} 
\caption{
Cooperative active sensing for 100 sparse landmarks and 15 clusters using 8 robots.
Top panel: Targets are denoted by circles, and clusters are shaded grey.
Squares represent the initial locations of the robots at the bottom of the panel.
Bottom panel: cluster assignment result from the online distributed auction for each of the eight robots.
Colors correspond to robot trajectories in the top panel.
Note that at time 150, the orange robot loses an auction for cluster 3 to the cyan robot, and selects a different cluster 7.
}
\label{fig:mmt}
\end{figure}

We present simulations of the proposed distributed estimation framework for multi-robot multi-target active localization in Fig.~\ref{fig:mmt}.
For these simulations, we uniformly randomly generated targets in a rectangle in $\reals^2.$
The communication range of the robots was set to 1500 m.
{
In this simulation, first we divide the targets into two sets of roughly equal size based on the prior location estimates $\set{\hbbx_i}_{i=1}^M.$
Then, we iteratively split the sets based on relative distance until the largest cluster has less than nine targets.
In our testing, we have found that a simple clustering strategy can be effectively generated using the prior over the target locations.
We leave the clustering strategy as a design choice in this algorithm that should be made once the location prior is available.
For example, in our example in Section~\ref{sec:st_results}, the sculpture garden can be naturally clustered with \emph{a priori} available tourist maps.
}

\section{Experiments}
\label{sec:exp}

In this section, we present experiments using a team of two ground robots (iRobot Creates){\blue , r1 and r2,} that localize a set of stationary targets.
The robots each carry a stereo rig mounted atop a servo that can rotate the rig $\pm 180^\circ.$
Each rig uses Point Grey Flea3\texttrademark\, cameras with resolution $1280\times1024$.
To simulate long distance localization, all images are downsampled by rate 24 so that the effective resolution is $54\times43$. 
Each robot is equipped with an on-board computer with 8GB RAM, an Intel Core i5-3450S processor and a 802.11n wireless network card for communication. 
All implementation is done in C++ and run on Robot Operating System (ROS).
We calibrate the intrinsic and the extrinsic parameters of the stereo rig offline using the Bouget Matlab\texttrademark\; camera calibration toolbox \cite{bouguet2004camera}.
For self-localization of the robots, our laboratory is equipped with an OptiTrack\texttrademark\, array of infrared cameras that act as a motion capture system.
The OptiTrack\texttrademark\, system tracks reflective markers that are rigidly attached to the robots.
Each robot can thus retrieve its own (and only its own) position and orientation by reading a ROS topic that is broadcast over wifi by a centralized computer.
Each robot also broadcasts auction bids and states to public ROS topics that are readable by the other robots for collaboration.
Additionally, the OptiTrack\texttrademark system provides us with the ground truth locations of the targets, which are colored ping pong balls, enabling us to directly check the accuracy of the landmark localization errors after the experiment.
The information function $\bbQ$ was derived using a statistical model of stereo vision, which we defer to Appendix \ref{app:noise}.

%

\begin{figure}[t]
\centering
\includegraphics[width=7cm]{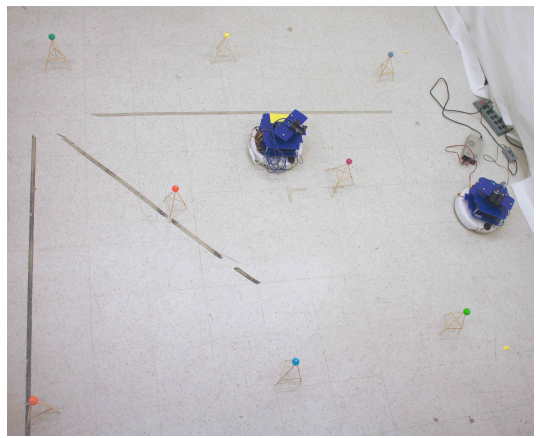}
\caption{Overhead photograph of the experimental setup
}
\label{fig:above}
\end{figure}
%
%

Fig.~\ref{fig:above} shows the experimental setup for the multi robot, multi landmark localization experiment.
We place eight colored ping pong balls in a square workspace of about 2 m side length for the multi robot, multi landmark localization experiment.
We use three types of local workspaces: section, semicircle, and full circular polar grids.
To avoid collisions between robots and ping pong balls and avoid overlapping local workspaces, the minimum and maximum radii were set to 30 and 50 cm for all local workspaces, respectively.
The available poses within each local workspace were located at five equally spaced radii between these two extrema.
The polar grids were also divided into increments of 15$^\circ$.
We set the covariance space parameters to $N_{\ccalL}=10,N_{\ccalA}=3, \kappa_\ccalL=10,$ and $  \kappa_\ccalA=7.$
We set $N_{\ccalT}$ to be twice the number of viewing angles in the local workspaces
For these experiments, we set $\rho=10^{-2}$ and $\gamma=0.9$.

In `transit' mode, a potential field algorithm guides the robots to the next local workspace and avoids collisions with landmarks.
All navigation and waypoint tracking relies on a PID controller using the next waypoint as the set point. 
In the experiment, robots generally came within 2 cm of their target waypoints. 
The servo guided the stereo cameras toward the estimate of the target locations with accuracy of $\pm 1^\circ$.
%
%
%
%
Cooperation of ${{\rm r1}}$ and ${{\rm r2}}$ relies on robots updating individual ROS topics and checking neighbors' ROS topics.
The ROS topic for a given robot contains that robot's current bid value and the subtask at which the bid is directed.

\begin{figure}[t]
\centering
\includegraphics[width=6cm]{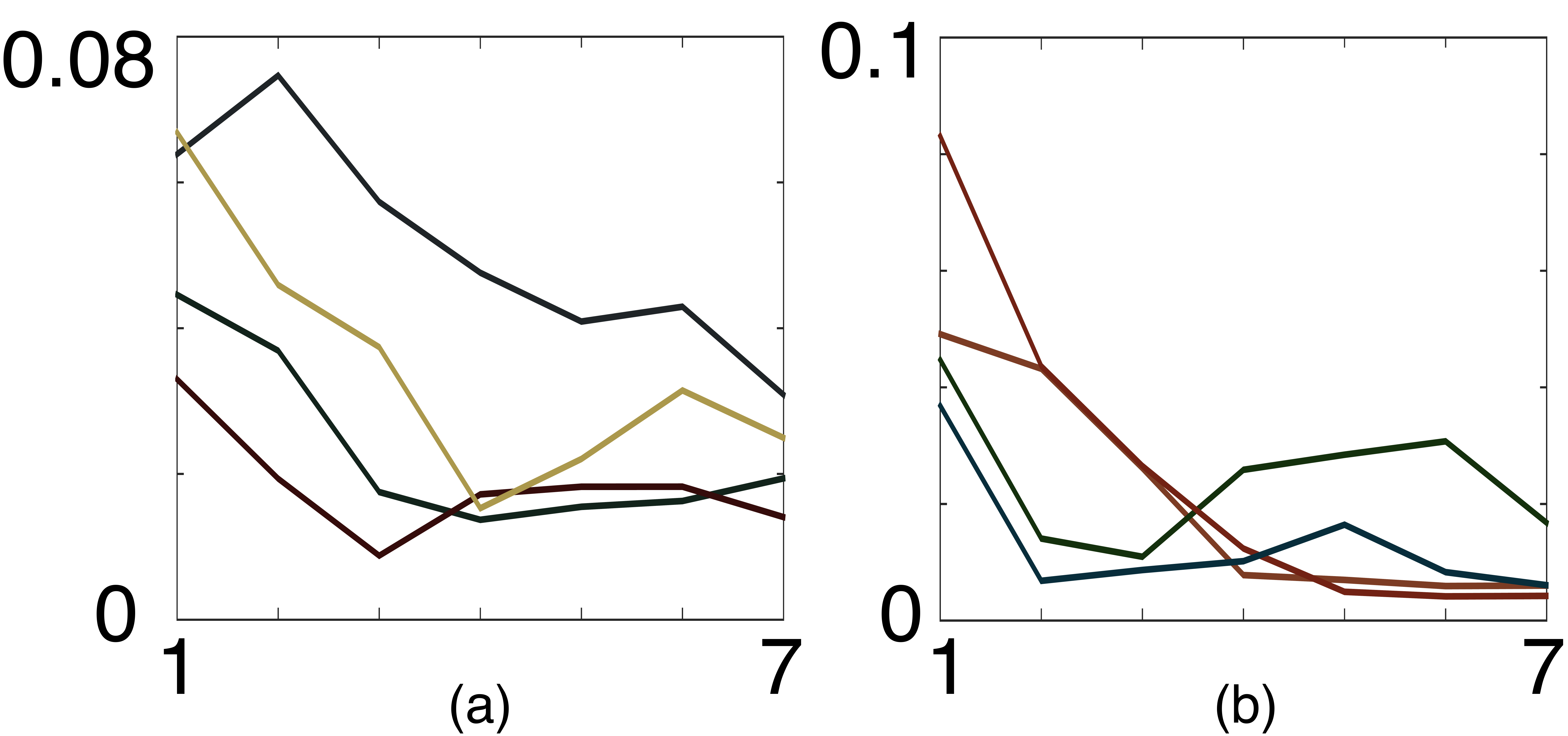}
\caption{
Plotting the filtered target localization error in meters for ${\rm r1}$ (a) and ${\rm r2}$ (b).
The horizontal axis is the number of observations for each.  This means that each robot took 28 observations total, localizing four targets each.
}
\label{fig:exptraj-err}
\end{figure}

The filtered estimates generated by observations along the paths followed by the robots in one run of the experiments are compared with the ground truth locations of the ping pong balls in Fig.~\ref{fig:exptraj-err}. 
The filtered errors are computed as the Euclidean distances between filtered estimates and the ground truth locations of the ping pong balls.
We note that any non-decreasing aspects of these plots are attributed to stochasticity in the dataset, and the fact that we are only approximating the true noise distribution with a data-driven Gaussian.
Note that, due to the relatively constrained space in our lab, the robots are not able to move as freely as in the simulations, which, in addition to unmodeled noise, is responsible for the smaller error reduction in Fig.~\ref{fig:exptraj-err} compared to the simulations.


\section{Conclusion}\label{sec:conc}In this paper, we addressed the task of estimating a finite set of hidden state vectors that have Gaussian priors using a mobile robotic sensor network.
We framed the problem using tools from optimal control, ultimately proposing a hierarchical Dynamic Programming solution.
{By including the covariance matrix in the local state-spaces then discretizing these spaces, we bound the computational complexity for a given error tolerance.
This is a desirable property compared to tree-based methods that actively explore the covariance space until the tolerance is reached.}
Then, we combined the local optimal trajectories in a cluster DP that balances between reducing the uncertainty in the hidden states and traveling among configuration spaces.
Our approach is still exponential in the number of hidden states per cluster, but the proposed hierarchical scheme reduces the base of the exponential to two without sacrificing the explicit dependence on the error covariance matrices of the objective function.
Then, we proposed distributed auction algorithm to divide the tasks of sensing each cluster among multiple robots.
Simulations and experiments on real robots show that the integrated multi-robot system can efficiently localize large groups of landmarks while remaining scalable, a novel pair of characteristics that POMDP and TSP approaches do not have.


\bibliographystyle{ieeetr}
\bibliography{charlie-refs}

\appendices

\section{Noise Model}
\label{app:noise}

In this paper, we assume that individual measurements are subject to a known zero mean Normal noise distribution $\bbnu \sim \normal{\bb0, \bbQ}$.
In what follows, we estimate $\bbQ$ so that this assumption holds for the stereo rigs used in our laboratory experiments.
This is critical for a variety of reasons:
\begin{itemize}
\item If the mean of $\bbnu$ is biased, then the KF will not converge to the ground truth.
\item If $\bbQ$ is an under approximation to the actual covariance of $\bbnu$, then the KF will become inconsistent and will not converge to the ground truth, if it converges at all.
\item If our choice of $\bbQ$ is too conservative, it may not be informative enough to be useful in the decision process at the core of the controller.
\end{itemize}
Making things even more difficult, we want to test the system in relatively extreme conditions, particularly at long ranges, when triangulation error distributions are known to be heavy tailed, biased away from zero, and highly asymmetric \cite{freundlich15cvpr}.


In our experiments we make use of the physics of stereo vision \cite{ma2012invitation}.
In particular, following \cite{matthies87} we assume that pixel error are Gaussian, and we propagate them to the localization estimates via triangulation equations, which we can use to give us an accurate distribution for $\bbnu$.
Specifically, if the robot registers a correspondence in the left (L) and right (R) cameras at pixel coordinates $[x_L, x_R, y]^\top$, then the coordinates of the 3-D point that generated the match are
\begin{equation}\label{eq:triangulate}
\mat{X\\Y\\Z} 
=  \frac{b}{d} \mat{\frac{1}{2}(x_L+x_R) \\ y \\ f},
\end{equation}
where $f$ is the focal distance, $b$ is the stereo baseline, $d = x_L - x_R$ is the disparity.
Let $\bbJ$ denote the Jacobian of \eqref{eq:triangulate}.
If the error in $[x_L, x_R, y]^\top$ has covariance matrix $\bbP,$ then the error covariance of $[X,Y,Z]^\top$ is $\bbJ \bbP \bbJ^\top.$
Moreover, if $\bbt$ and $\bbT$ are the translation vector and rotation matrix from the coordinate frame of the camera to a fixed coordinate frame, then the covariance of $\bbnu$ is $\bbT \bbJ \bbP \bbJ^\top \bbT^\top.$

With this noise propagation formula in mind, we now study errors in the observed pixels.
For this we follow a data-driven approach that we have recently developed in \cite{freundlich16nbv}.
Using a set of $n=600$ pairs of training images for each robot at various ranges and viewing angles, we obtain a regression that maps pixel observations $[x_L, x_R, y]^\top$ to a corrected tuple of pixels $[x_L^c, x_R^c, y^c]^\top$ such that the error in the corrected tuple is zero mean.
For every image pair, we project the ground truth landmark onto the image planes using the inverse mapping of \eqref{eq:triangulate}, giving us $\ell = 1, \dots, n$ individual output vectors $\bbY_\ell$, which we stack into an $n \times 3$ matrix of outputs $\bbY$.
The ground truth data for this training set includes a marker placed on top of the ping pong ball, and the pose information of the stereo rig, captured by a $\bbt $ and $\bbT .$
We then compute five features and, because the data are not centered, include one constant, for each raw pixel tuple according to the model
\begin{equation}\label{eq:model}
\bbX_\ell = \left[ 
1,\,  y_\ell,\, d_\ell,\,  x_{{L, \ell}}+  x_{{R, \ell}},\,  y  d_\ell, \frac{  x_{{L, \ell}}+  x_{{R, \ell}}}{d_\ell}
\right].
\end{equation} 
These features are taken from the constituent terms in the nonlinear equation \eqref{eq:triangulate}
Stacking the $\bbX_\ell$ into an $n \times 6$ matrix, we have a linear model $\bbY = \bbX \bbbeta + \bbepsilon,$ where $\bbbeta$ is a $6 \times 3$ matrix of coefficients and $\bbepsilon$ is an $n \times 3$ matrix of zero-mean Normal errors.
We experimentally verified that the rows of $\bbepsilon$ are roughly Gaussian, as can be seen in the right panel of Fig.~\ref{fig:px_errors}.
We refer to the raw pixels as \emph{uncorrected}.
The associated error vectors (computed with respect to the uncorrected pixels and the projected ground truth) $\bbepsilon_\ell^\text{uc}$ for $\ell = 1, \dots, n$ are plotted in the left panel of Fig~\ref{fig:px_errors}.
In the scatter plot it can be seen that the mean error is nonzero, contributing average bias to individual measurements.
Also note the apparent skew of the error distribution in the vertical ($y$) direction.

Applying the ordinary least squares estimator, the maximum likelihood estimate of the coefficient matrix is $\hat{\bbbeta} = (\bbX^\top \bbX)^{-1} \bbX^\top \bbY.$
Using $\hat{\bbbeta},$ the residual covariance in the pixel measurements for the two robots, named ${{\rm r1}} \; \and \; {{\rm r2}}$, we obtained are
\agn*{
\bbP_{\rm r1} = \mat{ 
    0.13&    0.09&    0.02\\
    0.09&    0.13&   -0.03\\
    0.02&   -0.03&    0.28}\!\!,
\bbP_{\rm r2} = \mat{     
    0.22&    0.16&    0.04\\
    0.16&    0.23&    0.03\\
    0.04&    0.03&    0.74}.
}
Note that the standard deviation of the $y$ pixel value, corresponding to the variances in the lower right entries of the above matrices, is 0.53 and 0.86 pixels, respectively.
This corresponds to errors in the height of the ping pong ball center in vertical coordinates.
The right panel of Fig.~\ref{fig:px_errors} shows the residual errors in the training set $\bbepsilon_\ell$ for $\ell = 1, \dots, n$ for the corrected vector $X \hat{\bbbeta}$ on ${\rm r1}$.

For prediction, if ${\rm r1}$ makes a new observation $(x^*_{L}, x^*_{R}, y^*)$, it forms a $1 \times 6$ vector $\bbX^*$.
Then, ${\rm r1}$ calculates the corrected pixels $[x_L^c, x_R^c, y^c]^\top = \bbX^* \hat{\bbbeta},$ which are subject to zero mean Normal errors. 
Using the corrected pixels, ${\rm r1}$ triangulates the relative location of the target via \eqref{eq:triangulate}, propagates $\bbP_{\rm r1}$ via the Jacobian $\bbJ,$ rotates and translates the estimates to global coordinates, and thus the assumptions that the error terms $\bbnu$ are zero mean with covariance $\bbQ= \bbT \bbJ \bbP_{\rm r1} \bbJ^\top \bbT^\top$ are approximately satisfied.
The case for ${\rm r2}$ is analogous.

\begin{figure}[t]
\centering
\begin{tabular}{c c}
\includegraphics[height=4cm]{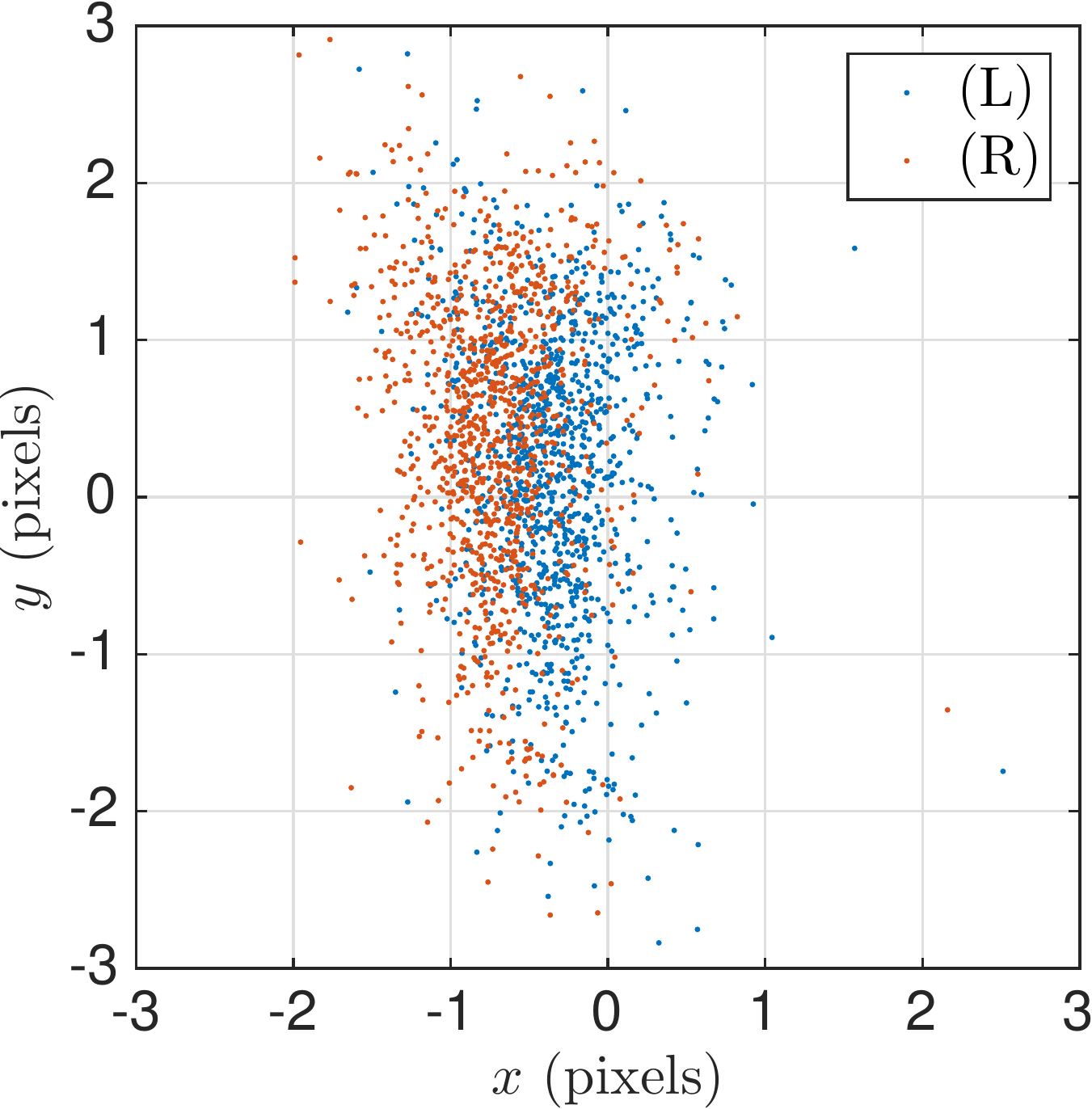}
&
\includegraphics[height=4cm]{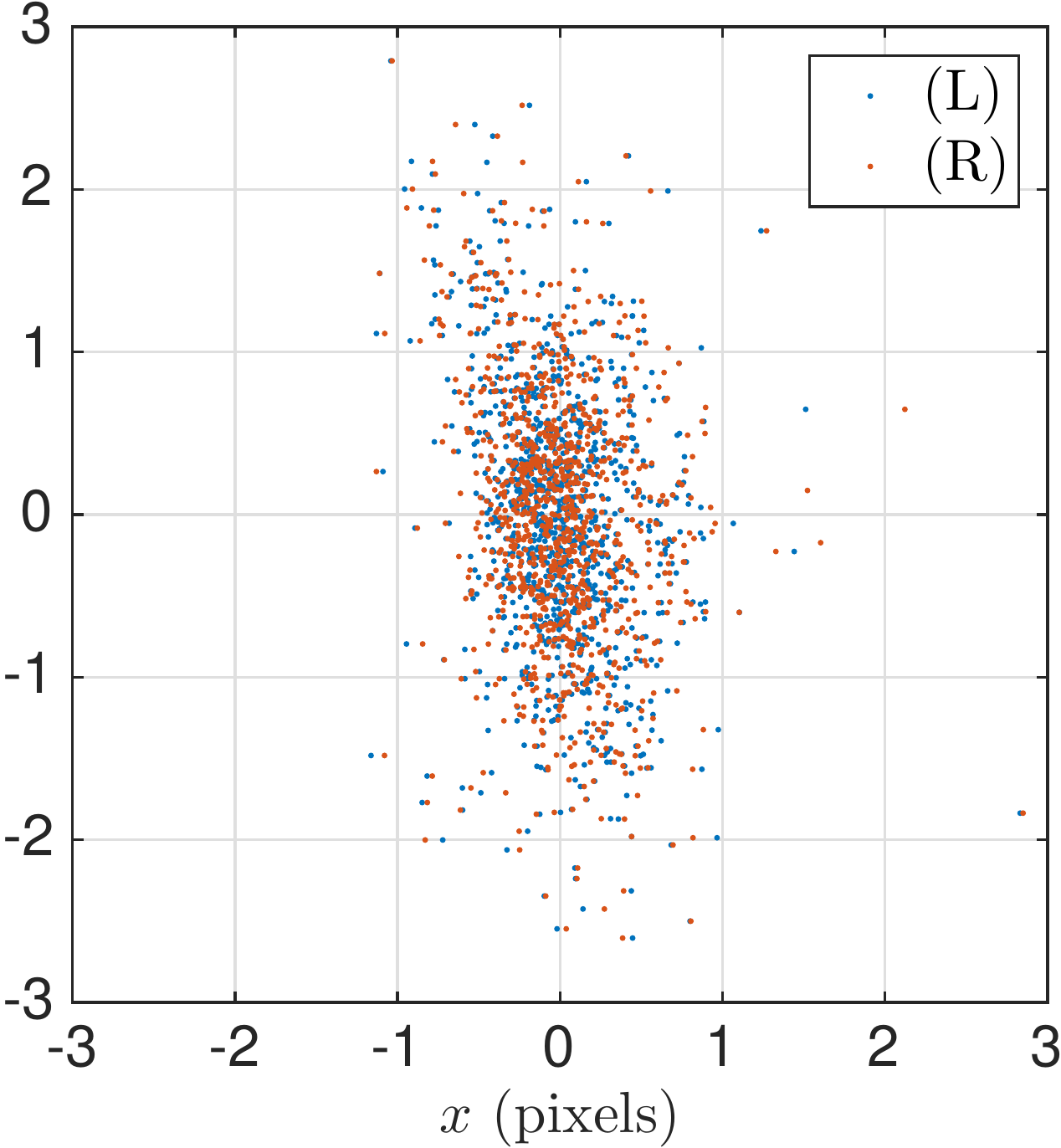}
\end{tabular}
\caption{Scatter plots of the residual errors $\bbepsilon_\ell^\text{uc}$ (left panel) and $\bbepsilon_\ell$ (right panel) for the training data for ${\rm r1}$
The plots for ${\rm r2}$ are similar.
}
\label{fig:px_errors}
\end{figure}

Finally, note that the regression takes place in three dimensions, whereas the algorithm is designed only for two dimensions.
For planning purposes, we consider only the components of $\bbQ$ that lie on the plane of the workspace, i.e., we do not use the third row and column for planning.
Of course, the experiments take place in three dimensions, so we need to use the full covariance matrices in the Kalman Filters.
Note that the algorithm proposed in this paper could be implemented in 3D at the expense of increasing the pose and covariance spaces accordingly.

\end{document}